%% file: exp-loss-implicit-reg-supp.tex
\documentclass{article}
\usepackage{jmlr2e}
\usepackage{natbib}
\usepackage{times}
\ifdefined\usebigfont

\usepackage[fontsize=13pt]{scrextend}
\usepackage[left=1.56in,right=1.56in,left=1.74in,right=1.74in]{geometry}
\else
\fi

\usepackage{xcolor}
\usepackage{microtype,booktabs,graphicx,subcaption}
\usepackage{amsmath,amssymb,amsthm,amsfonts,mathtools}
\usepackage{nicefrac,physics}
\usepackage{paralist,array}
\usepackage{hyperref}
\usepackage{thmtools,thm-restate}

\newcommand{\ind}[1]{{\mathbf{1}_{#1}}}
\renewcommand{\c}{\mathcal}
\newcommand{\cL}{\mathcal{L}}
\renewcommand{\b}{\mathbb}
\newcommand{\bR}{\mathbb{R}}
\newcommand{\w}{w}
\newcommand{\h}{G}
\newcommand{\p}{p}
\newcommand{\wo}{{\W{0}}}
\newcommand{\W}[1]{{\w_{(#1)}}}
\newcommand{\U}[1]{{U_{(#1)}}}
\newcommand{\WW}[1]{{W_{(#1)}}}
\newcommand{\Z}[1]{{Z_{(#1)}}}
\newcommand{\x}[1]{x_{#1}}
\newcommand{\y}[1]{y_{#1}}

\newcommand{\st}{\text{ s.t., }}

\newcommand{\innerprod}[2]{\left\langle #1,#2 \right\rangle}
\newcommand{\mybox}{\hfill\(\Box\)}

\DeclareMathOperator*{\argmin}{\arg\!\min}
\DeclareMathOperator*{\argmax}{\arg\!\max}

\renewcommand{\hat}{\widehat}
\renewcommand{\tilde}{\widetilde}

\newcommand{\din}{d}
\newcommand{\remove}[1]{{}}

\newtheorem{claim}{Claim}

\newtheorem{property}{Property}

\newtheorem{remark}{Remark}
\newtheorem{theorem}{Theorem}
\newtheorem{lemma}[theorem]{Lemma}

\newtheorem{proposition}[theorem]{Proposition}
\newtheorem{example}[theorem]{Example}
\newtheorem{subproposition}{Proposition}[theorem]

\begin{document}
\title{Characterizing Implicit Bias in Terms of Optimization Geometry}

\author{\name Suriya Gunasekar \addr TTI Chicago, USA \email suriya@ttic.edu   \AND \name Jason Lee  \addr USC Los Angeles, USA \email jasonlee@marshall.usc.edu \AND \name Daniel Soudry \addr Technion, Israel \email daniel.soudry@gmail.com  \AND Nathan Srebro \addr TTI Chicago, USA \email nati@ttic.edu}

\maketitle

\begin{abstract}
  We study the implicit bias of generic optimization methods, such as mirror
 descent, natural gradient descent, and steepest descent with respect
  to different potentials and norms, when optimizing underdetermined
  linear regression or separable linear classification problems.  We explore
  the question of whether the specific global minimum (among the many possible
  global minima) reached by an algorithm can be characterized in terms
  of the potential or norm of the optimization geometry, and independently of hyperparameter choices such as step-size and momentum.
\end{abstract}
\section{Introduction}\label{sec:intro}
Implicit bias from the optimization algorithm plays a crucial
role in learning  deep neural networks as it introduces effective capacity control not directly specified in the objective \citep{neyshabur2015search,neyshabur2015path,zhang2017understanding,  keskar2016large,wilson2017marginal,neyshabur2017geometry}. 
In overparameterized models where the
training objective has many global minima, optimizing using a specific  algorithm, such as  gradient descent, \textit{implicitly biases} the solutions to some special
global minima. The properties of the learned model, including its generalization performance, are thus crucially influenced by the choice of optimization algorithm used. In neural networks especially, characterizing these special global minima for  common algorithms such as stochastic gradient descent (SGD) is  essential for understanding  what the inductive bias of the learned model is and why such large capacity networks often show remarkably good generalization even in the absence of explicit regularization \citep{zhang2017understanding} or early stopping \citep{Hoffer2017}. 

Implicit bias from optimization depends on the choice of algorithm, and
changing the algorithm, or even changing associated
hyperparameter can change the implicit bias. 
For example, \citet{wilson2017marginal} showed that for 
some standard deep learning architectures, variants of SGD algorithm
with different choices of momentum and adaptive gradient updates
(AdaGrad and Adam) exhibit different biases and thus have different
generalization performance;
\citet{keskar2016large}, \citet{Hoffer2017} and \citet{Smith2018} study how the size of the
mini-batches used in SGD influences generalization; and
\citet{neyshabur2015path} compare the bias of path-SGD (steepest
descent with respect to a scale invariant path-norm) to standard SGD. 

It is therefore important to  explicitly relate different optimization algorithms to their implicit biases. 
Can we precisely characterize which global minima different  algorithms converge to? 
How does this depend on the loss function? 
What other choices including initialization,
step-size, momentum, stochasticity, and adaptivity, does the
implicit bias depend on? 
In this paper, we provide answers to some of these questions for simple linear regression and classification models. 
While neural networks are certainly more complicated than these simple linear models, the results here provide a segue into understanding such biases for more complex models. 

For linear models, we already have an understanding of the implicit bias of gradient
descent. For underdetermined least squares objective, gradient descent can be shown to converge to the minimum Euclidean norm solution. Recently, \citet{soudry2017implicit} studied gradient descent for linear
logistic regression.  The logistic loss is fundamentally different
from the squared loss in that the loss function has no attainable 
global minima.  Gradient descent iterates therefore diverge (the norm
goes to infinity), but \citeauthor{soudry2017implicit} showed that they
diverge in the direction of the hard margin support vector machine
solution, and therefore the decision boundary converges to this maximum
margin separator.   

Can we extend such characterization to other optimization methods
that work under different (non-Euclidean) geometries such as mirror
descent with respect to some potential, natural gradient descent with
respect to a Riemannian metric, and steepest descent with respect to
a generic norm? Can we relate the implicit bias to these geometries?

As we shall see, the answer depends on whether the loss function is
similar to a squared loss or to a logistic loss. This difference is captured
by two family of losses: \begin{inparaenum}[(a)]\item loss functions
  that have a unique finite root, like the squared loss and \item
  strictly monotone loss functions where the infimum is unattainable,
  like the logistic loss. \end{inparaenum}  For losses with a unique
finite root, we study the {\em limit point} of the optimization iterates,
$\w_\infty=\lim_{t\to\infty} \W{t}$.  For monotone losses, we study the
{\em limit direction}
$\bar{w}_\infty=\lim_{t\to\infty}\frac{\W{t}}{\|\W{t}\|}$.

In Section~\ref{sec:unique-root} we study linear models with loss functions that have unique finite
roots.  We obtain a robust characterization of the limit point for
mirror descent, and discuss how it is independent of step-size and
momentum.  For natural gradient descent, we show that the step-size does
play a role, but get a characterization for infinitesimal step-size.
For steepest descent, we show that not only does step-size affects the limit point, but even with infinitesimal step-size, the expected characterization does not hold.  The situation is fundamentally
different for strictly monotone losses such as the logistic loss
(Section~\ref{sec:monotonic}) where we do get a precise 
characterization of the limit direction for generic steepest descent.  We also study the adaptive gradient descent method (AdaGrad) \cite{duchi2011adaptive} (Section~\ref{sec:adagrad})
and optimization over matrix factorization  (Section~\ref{sec:mf}).  Recent studies
considered the bias of such methods for least squares problems
\citep{wilson2017marginal,gunasekar2017implicit}, and here we study
these algorithms for  monotone loss functions, obtaining a more robust characterization
for matrix factorization problems, while concluding that the implicit bias of AdaGrad depends on initial conditions including step-size even for strict monotone losses.  




\section{Losses with a Unique Finite Root}\label{sec:unique-root}
We first consider learning linear models using losses with a unique finite root, such as the squared loss, where the loss  $\ell(\hat{y},y)$ between a prediction $\hat{y}$ and label $y$ is minimized at a unique and finite value of $\hat{y}$. We assume without loss of generality, that $\min_{\hat{y}} \ell(\hat{y},y)=0$ and the unique minimizer is $\hat{y}=y$. 
\begin{property} [\textbf{Losses with a unique finite root}] \label{ass:finite-root}  For any $y$, a  sequence  $\{\hat{y}_t\}_{t=1}^\infty$ minimizes $\ell(.,y)$, i.e.,  $\ell(\hat{y}_t,y)\overset{t\to\infty}\longrightarrow\inf_{\hat{y}}\ell(\hat{y},y)=0$ if and only if $\hat{y}_t\overset{t\to\infty}\longrightarrow y$. 
\end{property}
Denote the training dataset $\{(\x{n},\y{n}):n=1,2,\ldots,N\}$ with features $\x{n}\in \b{R}^{\din}$ and labels $\y{n}\in \bR$. The empirical loss (or risk) minimizer of a linear model $f(x)=\innerprod{w}{x}$ with parameters $\w\in\b{R}^{\din}$ is given by,
\begin{equation}
\min_{\w} \c{L}(\w):=\sum_{n=1}^N {\ell}(\innerprod{\w}{\x{n}},\y{n}).
\label{eq:lm}
\end{equation}
We are particularly interested in the case where $N<d$ and the observations are realizable, i.e., $\min_w \c{L}(w)=0$. Under these conditions, the optimization problem in eq. \eqref{eq:lm} is underdetermined and has multiple global minima denoted by $\c{G}=\{\w:\c{L}(\w)=0\}=\{\w:\forall n,\;\innerprod{\w}{\x{n}}=\y{n}\}.$ Note that the set of global minima $\c{G}$ is the same for any loss $\ell$ with unique finite root (Property \ref{ass:finite-root}), including, e.g., the Huber loss, the truncated squared loss. 
Which specific global minima $\w\in\c{G}$ do different optimization algorithms reach when minimizing the empirical loss objective $\c{L}(w)$? 

\subsection{Gradient descent}\label{sec:gd-finite}
Consider  gradient descent updates for minimizing  $\c{L}(\w)$ with  step-size sequence $\{\eta_t\}_t$ and initialization $\W{0}$,
\begin{equation*}\W{t+1}=\W{t}-\eta_t\nabla\c{L}(\W{t}).\end{equation*}
If $\W{t}$ minimizes the empirical loss in eq.~\eqref{eq:lm}, then the iterates  converge to  the unique global minimum that is closest to initialization $\W{0}$ in $\ell_2$ distance, i.e.,  $\W{t}\to\argmin_{w\in\c{G}} \|w-\W{0}\|_2$. This can be easily seen as for any $\w$, the gradients $\nabla\c{L}(\w)=\sum_n{\ell}^\prime (\innerprod{\w}{\x{n}},\y{n})\x{n}$ are always constrained to the fixed  subspace spanned by the data $\{x_n\}_n$, and thus the iterates $\W{t}$ are confined to the low dimensional affine manifold $\wo+\text{span}(\{x_n\}_n)$. Within this  low dimensional manifold, there is a unique global minimizer $\w$ that satisfies the linear constraints in $\c{G}=\{\w:\innerprod{\w}{\x{n}}=\y{n},\forall n\in[N]\}$.

The same argument also extends  for updates with instance-wise stochastic gradients, where we use a stochastic estimate $\tilde{\nabla}\c{L}(\W{t})$ of the full gradient $\nabla\c{L}(\W{t})$ computed from a  random subset of instances  $S_t\subseteq[N]$,
\begin{equation}\label{eq:stoc}
 \tilde{\nabla}\c{L}(\W{t})=\sum\nolimits_{n\in S_{t}\subset[n]} \nabla_{\w} \ell(\innerprod{\W{t}}{\x{n_t}},\y{n_t}).
\end{equation}
Moreover, when initialized with $\wo=0$, 
the implicit bias characterization also extends to the following generic  momentum and acceleration based updates,
\begin{equation}\W{t+1}\!=\!\W{t}\!+\!\beta_t\Delta\W{t-1}\!-\!\eta_t\nabla\c{L}\!\left(\!\W{t}\!+\!\gamma_t\Delta\W{t-1}\!\right)\!,
\label{eq:mom}
\end{equation} 
where $\Delta\W{t-1}=\W{t}-\W{t-1}$. This includes Nesterov's acceleration ($\beta_t=\gamma_t$) \citep{nesterov1983method} and Polyak's heavy ball momentum ($\gamma_t=0$) \citep{polyak1964some}.


For losses with a unique finite root, the implicit bias of gradient descent therefore  depends only on the initialization and not on the step-size or momentum or mini-batch size. Can we get such  succinct characterization for other optimization algorithms? That is, characterize the bias in terms of the optimization geometry and initialization, but independent of choices of step-sizes, momentum, and stochasticity.
\subsection{Mirror descent}\label{sec:md-finite}
Mirror descent (MD) \citep{beck2003mirror,nemirovskii1983problem} was introduced as a generalization of gradient descent for  optimization over  geometries beyond the Euclidean geometry of gradient descent. In particular, mirror descent updates are defined for any strongly convex and differentiable potential  $\psi$ as
\begin{equation}
 \W{t+1}=\argmin_{\w\in\c{W}} \eta_t \innerprod{\w}{\nabla\c{L}(\W{t})}+D_\psi(\w,\W{t}),
 \label{eq:md-update}
\end{equation}
where $D_\psi(\w,\w')\!=\!\psi(\w)-\psi(\w')-\innerprod{\nabla \psi(\w')}{\w-\w'}$ is the \textit{Bregman divergence} \citep{bregman1967relaxation} w.r.t.  $\psi$, and $\c{W}$ is some constraint set for parameters $\w$. 

We first look at unconstrained optimization where $\c{W}=\bR^d$ and the update  in eq.~\eqref{eq:md-update} is equivalent to
\begin{equation}
\nabla\psi(\W{t+1})=\nabla\psi(\W{t})-\eta_t \nabla\c{L}(\W{t}).
\label{eq:md-upd-opt}
\end{equation}
For a strongly convex  potential $\psi$,  $\nabla\psi$ is called the link function and is invertible. Hence, the above updates are uniquely defined. Also, $\w$ and $\nabla\psi(\w)$ are referred as  \textit{primal} and \textit{dual} variables, respectively. 

Examples of potentials $\psi$ for mirror descent include the squared $\ell_2$ norm $\psi(w)=\nicefrac{1}{2}\|w\|_2^2$, which leads to  gradient descent; the entropy potential $\psi(w)=\sum_i w[i]\log{w[i]}-w[i]$; the spectral entropy for matrix valued $\w$, where $\psi(\w)$ is the entropy potential on the singular values of $\w$;  general quadratic potentials $\psi(w)=\nicefrac{1}{2}\|w\|_D^2=\nicefrac{1}{2}\,w^\top Dw$ for any positive definite matrix $D$; and the squared $\ell_p$ norms for $p\in(1,2]$.

From eq.~\eqref{eq:md-upd-opt}, we see that rather than the primal iterates $\W{t}$, it is the dual iterates $\nabla\psi(\W{t})$ that are constrained to the low dimensional data manifold $\nabla\psi(\W{0})+\text{span}(\{\x{n}\}_{n\in[N]})$. The arguments for gradient descent can now be generalized to get the following result. 
\begin{restatable}{theorem}{thmmdfinite} \label{thm:md-finite} For any loss $\ell$ with a unique finite root (Property~\ref{ass:finite-root}), any realizable dataset $\{\x{n},\y{n}\}_{n=1}^N$, and any strongly convex potential $\psi$, consider the mirror descent iterates $\W{t}$ from eq. \eqref{eq:md-upd-opt} for minimizing the empirical loss $\c{L}(\w)$ in eq. \eqref{eq:lm}. For all initializations $\W{0}$,  if the step-size sequence $\{\eta_t\}_t$ is chosen such that the limit point of the iterates $\w_\infty=\lim_{t\to\infty}\W{t}$ is a global minimizer of $\c{L}$, i.e., $\c{L}(\w_{\infty})=0$,  then $w_\infty$ is given by
\begin{equation}\label{eq:md_mnopt}
\w_{\infty}=\argmin_{\w:\forall n, \innerprod{\w}{\x{n}}=\y{n}} D_\psi(\w,\W{0}).
\end{equation}
\end{restatable}
In particular, if we start at $\wo=\argmin_\w \psi(w)$ (so that $\nabla\psi(\wo)=0$), then we get to $\w_\infty=\argmin_{\w\in\c{G}} \psi(\w)$, where recall that $\c{G}=\{\w:\forall n, \innerprod{\w}{\x{n}}=\y{n}\}$ is the set of global minima for $\c{L}(w)$. 

The analysis of Theorem~\ref{thm:md-finite} can also be extended for special cases of constrained mirror descent  (eq. \eqref{eq:md-update}) when $\c{L}(\w)$ is minimized over realizable affine equality constraints.
\begin{restatable}{subtheorem}{thmlinc} \label{thm:linc}
Under the conditions of Theorem~\ref{thm:md-finite}, consider constrained mirror descent updates $\W{t}$ from eq. \eqref{eq:md-update} with realizable affine equality constraints, that is $\c{W}=\{\w:G\w=h\}$ for some $G\in\b{R}^{d^\prime\times d}$ and $h\in\b{R}^{d^\prime}$ and additionally, $\exists w\in\c{W}$ with $\c{L}(w)=0$. For all initializations $\W{0}$, if the step-size sequence $\{\eta_t\}_t$ is chosen to asymptotically minimize $\c{L}$, i.e., $\c{L}(\w_{\infty})=0$,  then $w_\infty=\argmin_{\w\in\c{G}\cap \c{W}}D_\psi(\w,\W{0})$.  
\end{restatable}
For example,  in exponentiated gradient descent \citep{kivinen1997exponentiated}, which is mirror descent w.r.t $\psi(\w)=\sum_i \w[i]\log{\w[i]}-\w[i]$,  under the explicit simplex constraint   $\c{W}=\{\w:\sum_i\w[i]=1\}$, Theorem~\ref{thm:linc} shows that using uniform initialization $\wo=\frac{1}{d}\mathbf{1}$, mirror descent will return the  the maximum entropy solution ${\w_\infty=\argmin_{\w\in\c{G}\cap \c{W}} \sum_{i}\w[i]\log{\w[i]}}$. 

Let us now consider momentum for mirror descent. There are two possible generalizations of the gradient descent momentum in eq.~\eqref{eq:mom}: adding momentum either to primal variables $\W{t}$, or to  dual variables $\nabla\psi(\W{t})$, 
\begin{flalign}
&\text{Dual momentum:}&&\!\!\!\!\!\nabla\psi(\W{t+1})=\nabla\psi(\W{t})+\beta_t \Delta z_{(t-1)}-\eta_t \nabla\c{L}\left(\W{t}+\gamma_t \Delta w_{(t-1)}\right)\label{eq:dual-mom}&\\
&\text{Primal momentum:}&&\!\!\!\!\!\nabla\psi(\W{t+1})=\nabla\psi\big(\W{t}+\beta_t \Delta w_{(t-1)}\big)-\eta_t \nabla\c{L}\left(\W{t}+\gamma_t \Delta w_{(t-1)}\right)\label{eq:primal-mom}&
\end{flalign}
where $\Delta z_{(-1)}=\Delta w_{(-1)}=0$, and for $t\ge1$, $\Delta z_{(t-1)}=\nabla\psi(\W{t})-\nabla\psi(\W{t-1})$ and $\Delta w_{(t-1)}=\W{t}-\W{t-1}$ are the momentum terms in the primal and dual space, respectively; and $\{\beta_t\ge 0,\gamma_t\ge 0\}_t$ are the momentum parameters.

If we initialize at $\wo=\argmin_\w \psi(w)$, then  even with dual momentum  $\nabla\psi(\W{t})$  continues to remain in the data manifold. This leads to the following extension of Theorem~\ref{thm:md-finite}.
\begin{restatable}{subtheorem}{thmmdfinitea} \label{thm:md-finite1a}Under the conditions in Theorem~\ref{thm:md-finite}, if initialized at  $\wo=\argmin_\w \psi(w)$, then the mirror descent updates with dual momentum also converge to \eqref{eq:md_mnopt}, i.e., for all $\{\eta_t\}_t,\{\beta_t\}_t,\{\gamma_t\}_t$, if $\W{t}$ from eq. \eqref{eq:dual-mom} converges to $\w_{\infty}\in\c{G}$, then $\w_{\infty}=\argmin_{\w\in\c{G}}\psi(\w)$. 
\end{restatable}
\begin{restatable}{remark}{remsgd}\label{rem:sgd}Following the same arguments, we can show that Theorem~\ref{thm:md-finite}--\ref{thm:md-finite1a} also hold when instancewise stochastic gradients defined in eq.  \eqref{eq:stoc} are used in place of $\nabla\c{L}(\W{t})$.
\end{restatable}
{\small
\begin{figure*}
    \centering
    \begin{subfigure}[b]{0.35\textwidth}
        \includegraphics[width=\textwidth]{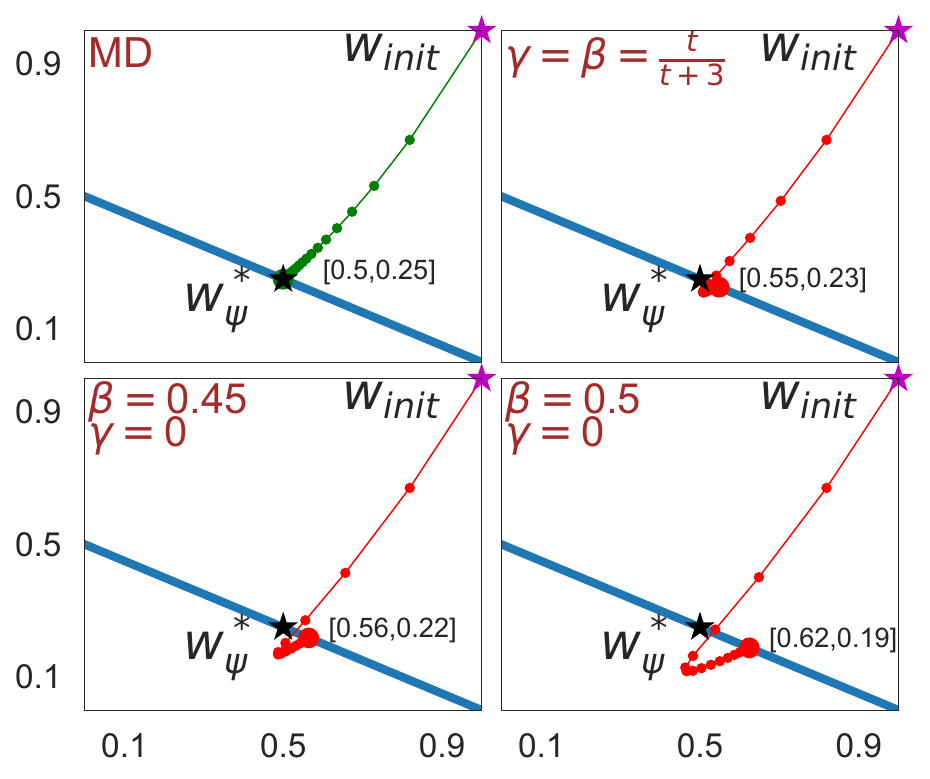}
        \caption{\label{fig:md}\centering Mirror descent primal \newline momentum (Example~\ref{ex:md})}
    \end{subfigure}
    \begin{subfigure}[b]{0.27\textwidth}
        \includegraphics[width=\textwidth]{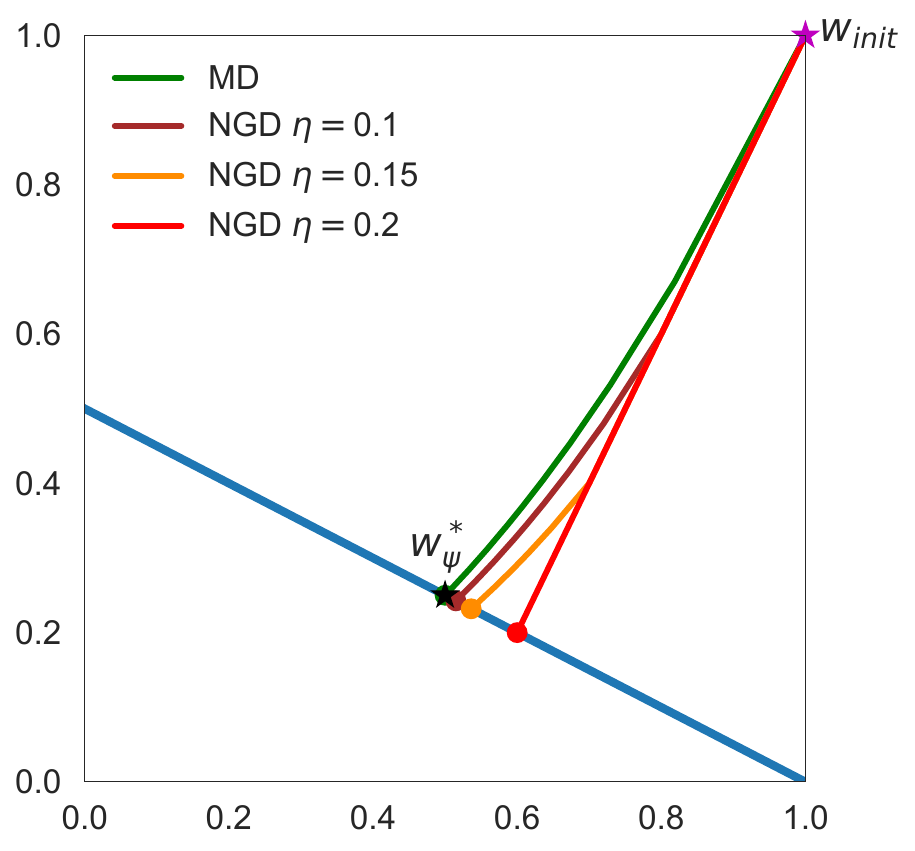}
        \caption{\label{fig:ngd}\centering Natural gradient descent\newline (Example~\ref{ex:ngd})}
    \end{subfigure}
    \begin{subfigure}[b]{0.35\textwidth}
        \includegraphics[width=\textwidth]{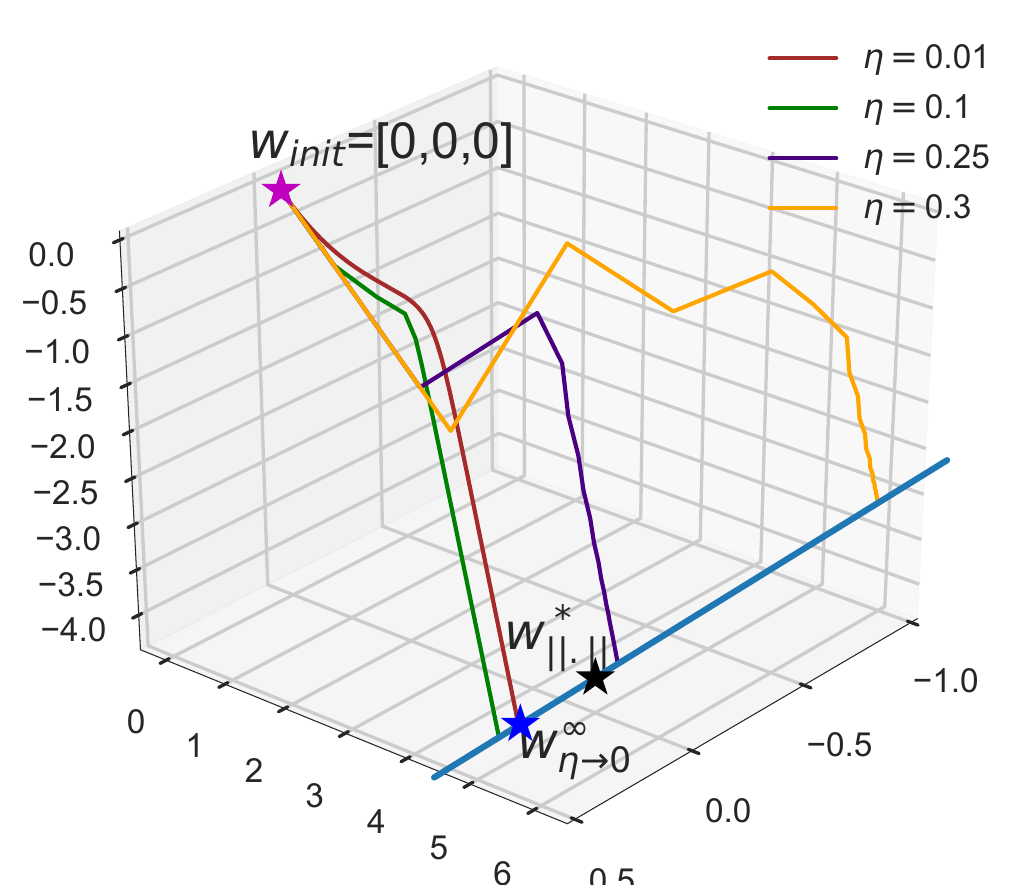}
        \caption{\label{fig:sd}\centering Steepest descent w.r.t $\|.\|_{\nicefrac{4}{3}}$ \newline  (Example~\ref{ex:sd})}
    \end{subfigure}
    \caption{\label{fig:finite-root}\small Dependence of implicit bias on step-size and momentum: In $(a)$--$(c)$, the blue line denotes the set $\c{G}$ of global minima for the respective examples. In $(a)$ and $(b)$, $\psi$ is the entropy potential and all algorithms are initialized with $\wo=[1,1]$ so that $\psi(\wo)=\argmin_{w}\psi(w)$. $\w^*_\psi=\argmin_{\psi\in\c{G}}\psi(w)$ denotes the minimum potential global minima we expect to converge to. 
$(a)$ \textbf{Mirror descent with primal momentum (Example~\ref{ex:md}):}  the global minimum that  eq. \eqref{eq:primal-mom}  converges to depends on the momentum parameters---the sub-plots contain the trajectories of eq. \eqref{eq:primal-mom} for different choices of $\beta_t=\beta$ and $\gamma_t=\gamma$. 
$(b)$ \textbf{Natural gradient descent  (Example~\ref{ex:ngd}):}  for different  step-sizes $\eta_t=\eta$, eq. \eqref{eq:ngd-update} converges to different global minima. Here, $\eta$ was chosen to be small enough to ensure $\W{t}\in\text{dom}(\psi)$.  (c) \textbf{Steepest descent w.r.t $\|.\|_{\nicefrac{4}{3}}$ (Example~\ref{ex:sd}):} the global minimum to which eq. \eqref{eq:sd-update} converges to depends on $\eta$. Here $\wo=[0,0,0]$,  $\w^*_{\|.\|}=\argmin_{\psi\in\c{G}}\|w\|_{\nicefrac{4}{3}}$ denotes the minimum norm global minimum, and $\w^\infty_{\eta\to0}$ denotes the solution of infinitesimal SD with $\eta\to0$.  Note that even as $\eta\to 0$, the expected characterization does not hold, i.e., $\w^\infty_{\eta\to0}\neq \w^*_{\|.\|}$.}
\end{figure*}
}

Let us now look at primal momentum. 
For general  potentials $\psi$, the dual  iterates $\nabla\psi(\W{t})$ from the primal momentum  can fall off the data  manifold and the additional components influence the final solution. Thus,  the specific global minimum that  the iterates $\W{t}$  converge to  will depend on the values of  momentum parameters $\{\beta_t,\gamma_t\}_t$ and  step-sizes $\{\eta_t\}_t$ as demonstrated in  the following example.  
\begin{example}\label{ex:md}
Consider optimizing $\c{L}(w)$  with dataset $\{(\x{1}=[1,2], \y{1}=1)\}$ and squared loss $\ell(u,y)=(u-y)^2$ using primal momentum updates from eq. \eqref{eq:primal-mom} for MD w.r.t. the entropy potential $\psi(w)=\sum_i w[i]\log{w[i]}-w[i]$. For initialization $\wo=\argmin_\w \psi(w)$, Figure~\ref{fig:md} shows  how different choices of momentum  $\{\beta_t,\gamma_t\}$ change the limit point $w_\infty$. Additionally, we show the following:
\begin{subproposition} \label{prop:md-primal} In Example~\ref{ex:md}, consider the case where  primal momentum is  used only in the first step, but $\gamma_t=0$ and $\beta_{t}=0$ for all $t\ge2$. For any $\beta_1>0$, there exists $\{\eta_t\}_t$, such that $\W{t}$ from \eqref{eq:primal-mom}  converges to a global minimum, but not to $\argmin_{\w\in\c{G}}\psi(\w)$. 
\end{subproposition}
\end{example}


\subsection{Natural gradient descent}
 Natural gradient descent  (NGD)  was introduced by \citet{amari1998natural} as a modification of  gradient descent, wherein the updates are chosen to be the steepest descent direction w.r.t a Riemannian metric tensor $H$ that maps $w$ to a positive definite local metric $H(w)$. The updates are given by,
\begin{equation}
\W{t+1}= \W{t}-\eta_t H\big(\W{t}\big)^{-1}{\nabla{\c{L}}}(\W{t}).
\label{eq:ngd-update}
\end{equation}
In many instances, the metric tensor $H$ is specified by the Hessian $\nabla^2\psi$ of a strongly convex potential $\psi$. For example, when the metric over the Riemannian manifold is the KL divergence between distributions $P_w$ and $P_{w'}$ parameterized by $w$, the metric tensor is given by $H(w)=\nabla^2\psi(P_w)$, where the potential $\psi$ is the entropy potential over $P_w$.

\paragraph{Connection to mirror descent} When $H(w)=\nabla\psi^2(w)$ for a strongly convex potential $\psi$, 
as the step-size $\eta$ goes to zero,  the iterates $\W{t}$ from natural gradient descent in eq. \eqref{eq:ngd-update} and mirror descent  w.r.t $\psi$ in eq. \eqref{eq:md-update} converge to each other, and the common dynamics  in the limit is given by,
\begin{equation}
\dv{\nabla \psi(\W{t})}{t}=-{\nabla{\c{L}}}(\W{t})\implies\dv{\W{t}}{t}=-\nabla^2\psi(\W{t})^{-1}{\nabla{\c{L}}}(\W{t}).
\label{eq:md-ode}
\end{equation}
Thus, as the step-sizes are made infinitesimal, the limit point of  natural gradient descent $w_\infty= \lim_{t\to\infty}\W{t}$ is also the limit point of  mirror descent  and hence will be biased towards solutions with minimum divergence to the initialization, i.e., as $\eta\to0$, $w_\infty=\argmin_{w\in\c{G}} D_\psi(w, \W{0})$.

For general step-sizes $\{\eta_t\}$, if the potential $\psi$ is quadratic, $\psi(w)=\nicefrac{1}{2}\|w\|_D^2$ for some positive definite $D$, we get linear link functions  $\nabla\psi(w)=Dw$ and  constant metric tensors $\nabla^2\psi(w)=H(w)=D$, and  the natural gradient descent updates \eqref{eq:ngd-update} are the same as the mirror descent  \eqref{eq:md-upd-opt}. Otherwise the updates in eq. \eqref{eq:ngd-update} is only an approximation of the mirror descent update $\nabla\psi^{-1}(\nabla\psi(\W{t})-\eta_t\nabla\c{L}(\W{t}))$. 


For natural gradient descent with finite step-size and non-quadratic potentials $\psi$, the characterization in eq. \eqref{eq:md_mnopt} generally does not hold. We can  see this as for any initialization $\W{0}$,  a finite $\eta_1>0$  will  lead to $\W{1}$ for which the dual variable $\nabla\psi(\W{1})$ is no longer in the data manifold $\text{span}(\{\x{n}\})+\nabla\psi(\W{0})$, and hence will converge to a different global minimum dependent on the step-sizes  $\{\eta_t\}_t$. 
\begin{example}\label{ex:ngd} Consider optimizing $\c{L}(w)$ with squared loss over dataset $\{(\x{1}=[1,2], \y{1}=1)\}$ using the natural gradient descent  w.r.t. the metric tensor given by  $H(w)=\nabla^2\psi(w)$, where  $\psi(w)=\sum_iw[i]\log{w[i]}-\w[i]$,  and initialization $\wo=[1,1]$. Figure~\ref{fig:ngd} shows that NGD with different step-sizes $\eta$ converges to different global minima. For a simple analytical example:  take one finite step $\eta_1>0$ and then follow the continuous time path in eq. \eqref{eq:md-ode}.
\begin{subproposition}\label{prop:ngd} For almost all  $\eta_1>0$, $\lim_{t\to\infty}\W{t}=\argmin_{\w\in\c{G}}D_\psi(w,\W{1})\neq \argmin_{\w\in\c{G}}D_\psi(w,\W{0})$.
\end{subproposition}
\end{example}
\remove{
\paragraph{Necessary conditions: }
So far, we saw conditions under which MD w.r.t $\psi$ and NGD w.r.t. $\nabla^2\psi$ converges to the global optimum that minimizes the Bregman divergence to the initialization. 
We can also ask the inverse question of what kind of updates will lead to  an implicit bias towards minimizing the Bregman divergence to an initialization.

\begin{lemma}\label{lem:cond}Consider any strongly convex and differentiable function $\phi:\bR^d\to \bR$ and the associated Bregman divergence $D_\phi(\w,\w')=\phi(\w)-\phi(\w')-\innerprod{\nabla \phi(\w')}{\w-\w'}$. 

For any $\ell(u,y)$ with finite roots and sequence of  $\W{t}$ that converges to a global optimum of $\c{L}(\w)$ in \eqref{eq:lm}, we get $\lim_{t\to\infty}\W{t}= \argmin_{\w\in\c{G}} D_\phi(\w,\W{0})$ for \textit{every} $\W{0}$ if and only if $\forall t, \nabla\phi(\W{t+1})-\nabla\phi(\W{t})\in\text{span}(\{\x{n}\})$. 
\end{lemma}
In particular, for first order algorithms where $\W{t}$ is updated only based on past updates $\{\W{t'}\}_{t'<t}$ and gradients $\{\c{L}(\W{t'})\}_{t'<t}$, to get the implicit bias of $\argmin_{\w\in\c{G}} D_\phi(\w,\W{0})$ for every $\W{0}$, we need to necessarily use  updates of the form $ \nabla\phi(\W{t+1})=\nabla\phi(\W{t})+\sum_{t'<t} \xi_{t'}\nabla\c{L}(\W{t})$ and  its stochastic variants, for some $\{\xi_{t'}\in\bR\}$---essentially minor variants of MD.

The above results do not preclude cases where we can get implicit regularization under additional conditions on $\W{0}$, for example,  in Theorem~\ref{thm:md-finite1a} we get implicit regularization to $\psi(w)$ for MD with dual momentum when $\nabla\psi(\W{0})=0$, but this does not extend to \textit{any} other initialization, and hence is not covered by Lemma~\ref{lem:cond}. }

\subsection{Steepest Descent}
Gradient descent is also a special case of  steepest descent (SD) w.r.t  a generic norm $\|.\|$ \citep{boyd2004convex} with updates  given by, 
\begin{equation}
\W{t+1}= \W{t}+\eta_t \Delta \W{t}, \text{ where }\Delta \W{t}=\argmin_{v} \innerprod{\nabla{\c{L}}(\W{t})}{v}+\frac{1}{2}\|v\|^2.
\label{eq:sd-update}
\end{equation}
The optimality of $\Delta\W{t}$ in eq. \eqref{eq:sd-update} requires $-\nabla{\c{L}}(\W{t})\in \partial\|\Delta\W{t}\|^2$, which is equivalent to,
\begin{equation}
\langle\Delta\W{t},\!-\nabla\c{L}(\W{t})\rangle\!=\!\|\Delta\W{t}\|^2\!=\!\|\nabla\c{L}(\W{t})\|_\star^2. 
\label{eq:sd-upd-opt}
\end{equation}
Examples of steepest descent include gradient descent, which is
steepest descent w.r.t $\ell_2$ norm and coordinate descent, which is
steepest descent w.r.t $\ell_1$ norm. In general, the update
$\Delta\W{t}$ in eq. \eqref{eq:sd-update} is not uniquely defined and
there could be multiple direction $\Delta\W{t}$ that minimize eq.
\eqref{eq:sd-update}.  In such cases, any minimizer of eq.  \eqref{eq:sd-update}
is a valid steepest descent update and satisfies eq.
\eqref{eq:sd-upd-opt}.

Generalizing gradient descent, we might expect the limit point $\w_\infty$ of steepest
descent w.r.t an arbitrary norm $\|.\|$ to be the solution closest to initialization in corresponding norm, 
$\argmin_{\w\in\c{G}} \|\w-\wo\|$. This is indeed the case for
quadratic norms $\|v\|_D=\sqrt{v^\top Dv}$ when eq.~\ref{eq:sd-update} is equivalent to mirror descent 
with $\psi(w)=\nicefrac{1}{2}\|w\|_D^2$. Unfortunately, this
 does not hold for general norms.
\begin{example}\label{ex:sd}
Consider minimizing $\c{L}(w)$ with  dataset $\{(\x{1}=[1,1,1], \y{1}=1), (\x{2}=[1,2,0], \y{2}=10)\}$ and loss $\ell(u,y)=(u-y)^2$ using steepest descent updates  w.r.t. the $\ell_{4/3}$ norm. The empirical results for this problem in Figure~\ref{fig:sd} clearly show that even for $\ell_p$ norms where the $\|.\|^2_p$ is smooth and strongly convex, the corresponding steepest descent converges to a global minimum that depends on the step-size. Further, even in the continuous step-size limit of $\eta\to 0$, $\W{t}$ does not converge to $\argmin_{\w\in\c{G}} \|\w-\wo\|$. 
\end{example}

\paragraph{Coordinate descent}
Steepest descent w.r.t.~the $\ell_1$ norm is called the coordinate descent, with updates:
\begin{equation*}
\Delta\W{t+1}\in\text{conv}\left\{-\eta_t \frac{\partial \c{L}(\W{t})}{\partial \w[j_t]}e_{j_t}:j_t=\argmax_j \left|\frac{\partial \c{L}(\W{t})}{\partial \w[j]}\right|\right\},
\label{eq:sdl1-finite}
\end{equation*}
where $\text{conv}(S)$ denotes the convex hull of the set $S$, and $\{e_j\}$ are the standard basis, i.e., when multiple partial derivatives are maximal, we can choose any convex combination of the maximizing coordinates, leading to many possible coordinate descent optimization paths.

The connection between optimization paths
of coordinate descent and the $\ell_1$ \textit{regularization path} given by, 
$\hat{w}(\lambda)=\argmin_{w}\c{L}(w)+\lambda \|w\|_1$, has been studied by \citet{efron2004least}.  The specific coordinate descent path where updates are along the average of all optimal coordinates and the  step-sizes are infinitesimal is equivalent to
forward stage-wise selection, a.k.a.~$\epsilon$-boosting
\citep{friedman2001greedy}.  When the $\ell_1$ regularization path
$\hat{w}(\lambda)$ is monotone in each of the coordinates, it is
identical to this stage-wise selection path, i.e.,~to a coordinate
descent optimization path (and also to the related LARS path)
\citep{efron2004least}. In this case, at the limit of $\lambda\to0$ and $t\to\infty$,  the optimization and regularization paths, both converge to the  minimum $\ell_1$
norm solution.  However, when the regularization path $\hat{w}(\lambda)$ is not
monotone, which can and does happen, the optimization and regularization paths diverge, and forward stage-wise selection can converge to solutions with sub-optimal $\ell_1$ norm.
This matches our understanding that steepest descent w.r.t.~a norm $\|.\|$, in
this case the $\ell_1$ norm might converge to a solution that is {\em  not} always the minimum $\|.\|$  norm solution.

\subsection{Summary for losses with a unique finite root}
For losses with a unique finite root, we
characterized the implicit bias of generic mirror descent algorithm in
terms of the potential function and initialization. This
characterization extends for momentum in the dual space as well as to
natural gradient descent in the limit of infinitesimal step-size. We also saw that the characterization breaks for mirror
descent with primal momentum and natural gradient descent with finite
step-sizes. Moreover, for steepest descent with general norms, we were
unable to get a useful characterization even in the infinitesimal step
size limit. In the following section, we will see that for
strictly monotone losses, we {\em can} get a characterization also for
steepest descent.

\section{Strictly Monotone Losses}\label{sec:monotonic} 
We now turn to strictly monotone loss functions $\ell$ where the behavior of the implicit bias is fundamentally different, and as are the situations when the implicit bias can be characterized. Such losses are common in classification problems where $y=\{-1,1\}$ and $\ell(f(x),y)$ is typically a continuous surrogate of the $0$-$1$ loss. Examples of such losses include logistic loss, exponential loss, and probit loss. 
\begin{property} [\textbf{Strict monotone losses}] \label{ass:mon-bdd}   
$\ell(\hat{y},y)$ is bounded from below, and $\forall y$, $\ell(\hat{y},y)$ is strictly monotonically decreasing in  $\hat{y}$.
 Without loss of generality,  $\forall y$,  $\inf_{\hat{y}} \ell(\hat{y},y)= 0$ and $\ell(\hat{y},y)\overset{\hat{y} y\to \infty}\longrightarrow 0$. 
\end{property}
 
We look at classification models that fit the training data $\{\x{n},\y{n}\}_n$ with linear decision boundaries  $f(x)=\innerprod{w}{x}$ with decision rule  given by $\hat{y}(x)=\text{sign}(f(x))$.  
 In many instances of the proofs, we also assume without loss of generality that $y_n=1$ for all $n$, since for linear models, the sign of $\y{n}$ can equivalently be absorbed into $\x{n}$. 
 

We again look at unregularized empirical risk minimization objective of the form in eq. \eqref{eq:lm}, but now with strictly monotone losses. 
When the training data $\{\x{n},\y{n}\}_n$ is not linearly separable, the empirical  objective $\c{L}(\w)$  can have a finite global minimum. However, if the dataset is linearly separable, i.e., $\exists w:\forall n, \y{n}\innerprod{w}{\x{n}}>0$, the  empirical loss $\c{L}(\w)$ is again ill-posed, and moreover $\c{L}(\w)$ does not have any finite minimizer, i.e, $\c{L}(w)\to0$ only as $\|w\|\to\infty$. 
Thus, for any sequence $\{\W{t}\}_{t=0}^\infty$,  if  $\c{L}(\W{t})\to 0$,  then  $\W{t}$ necessarily diverges to infinity rather than converge, and hence we cannot talk about $\lim_{t\to\infty}\W{t}$. Instead, we look at the limit direction $\bar{\w}_{\infty}=\lim\limits_{t\to\infty}\frac{\W{t}}{\|\W{t}\|}$ whenever the limit exists. We refer to existence of this limit as convergence in direction. Note that, the limit direction fully specifies the decision rule of the classifier that we care about. 

\remove{To show convergence of ${\W{t}}$ in direction, we further restrict to  losses with exponential tails. 
\begin{property} [\textbf{Tight exponential tail}] \label{ass:exp-tail} $\ell(u,y)$ has a tight exponential tail if $\exists \mu>0$ and $u_0>0$ such that 
$$\forall yu>u_0,\;\;(1-e^{-\mu yu})e^{-yu}\le -{\ell}^\prime (u,y)\le (1+e^{-\mu yu})e^{-yu}$$
\end{property}
This includes, exponential, logistic, and sigmoid losses. \remove{chk}More specifically, our results are proved only for the case of exponential loss $\ell(u,y)=\exp(-uy)$. But generalization to tight exponential tails can be obtained for with bit of additional algebra along the lines of \citet{soudry2017implicit} and \citet{telgarsky2013margins}. }

We focus on the  exponential loss $\ell(u,y)=\exp(-uy)$. However,  our results can be extended to  loss functions with tight exponential tails, including logistic and sigmoid losses, along the lines of \citet{soudry2017implicit} and \citet{telgarsky2013margins}. 




\subsection{Gradient descent} 
\citet{soudry2017implicit} showed that for almost all linearly separable datasets, gradient descent  with \textit{any initialization and any bounded step-size}  converges in direction to maximum margin separator with unit $\ell_2$ norm, i.e., the hard margin support vector machine classifier,
  \[
  \bar{w}_\infty=\lim_{t\to\infty}\frac{\W{t}}{\|\W{t}\|_2}= w^*_{\|.\|_2}:=\argmax_{\|\W{t}\|_2\le 1} \min_n \y{n}\innerprod{\w}{\x{n}}.\]  

This characterization of the implicit bias  is independent of both the step-size as well as the initialization. We already see a fundamentally difference from the implicit bias of gradient descent for losses with a unique finite root (Section~\ref{sec:gd-finite}) where the characterization  depended on the initialization. 

Can we similarly characterize the implicit bias of different algorithms establishing $\W{t}$ converges in direction and calculating $\bar\w_\infty$? Can we do this even when we \textit{could not} characterize the limit point $\w_\infty=\lim_{t\to\infty} \W{t}$ for losses with unique finite roots? As we will see in the following section, we can indeed answer these questions for steepest descent w.r.t arbitrary norms. 

\subsection{Steepest Descent}\label{sec:sd-finite}
Recall that for squared loss, the limit point of steepest descent
 depends on the step-size, and we were unable obtain a useful 
characterization even for infinitesimal step-size and zero initialization.  In contrast, for exponential loss,  the following
theorem provides a crisp characterization of the limit direction of
steepest descent as a maximum margin solution, independent of 
 step-size (as long as it is small enough) and  initialization. Let $\|.\|_\star$ denote the dual norm of $\|.\|$. 
\begin{restatable}{theorem}{thmsdexp} \label{thm:sd-exp}
For any separable dataset $\{\x{n},\y{n}\}_{n=1}^N$ and any norm $\lVert \cdot \rVert$, consider the steepest
descent updates  from eq. \eqref{eq:sd-upd-opt} for minimizing $\c{L}(\w)$ in eq. \eqref{eq:lm} with the exponential loss $\ell(u,y)=\exp(-uy)$. For all initializations  $\W{0}$, and all bounded step-sizes satisfying $\eta_t \le \min\{\eta_+, \frac{1}{B^2\c{L}(\W{t})}\}$, where $B:=\max_n \|x_n\|_\star$ and $\eta_+<\infty$ is any finite upper bound, the iterates $\W{t}$ satisfy the following,
\[ \lim_{t\to\infty} \min_n\frac{\y{n}\innerprod{\W{t}}{x_n}}{\|\W{t}\|}= \max_{\w:\|w\|\le 1} \min_{n} {\y{n}\innerprod{\w}{x_n}}.\]
In particular, if there is a unique maximum-$\|.\|$ margin  solution $\w^\star_{\|.\|}=\argmax_{\w:\|\w\|\le1} \min_{n} {\y{n}\innerprod{\w}{x_n}}$, then the limit direction is given by  $\bar{w}_\infty=\lim\limits_{t\to\infty}\frac{\W{t}}{\|\W{t}\|}=\w^\star_{\|.\|}$.
\end{restatable}


A special case of Theorem~\ref{thm:sd-exp} is for steepest descent
w.r.t.~the $\ell_1$ norm, which as we already saw corresponds to
coordinate descent.  More specifically, coordinate descent on the
exponential loss can be thought of as an alternative presentation of
AdaBoost \citep{schapire2012boosting}, where each coordinate represents
the output of one ``weak learner''.  Indeed, initially mysterious
generalization properties of boosting have been understood in terms of
implicit $\ell_1$ regularization \citep{schapire2012boosting}, and
later on AdaBoost with small enough step-size was shown to converge in
direction precisely to the maximum $\ell_1$ margin solution \citep{zhang2005boosting,shalev2010equivalence,telgarsky2013margins}, just as
 guaranteed by Theorem~\ref{thm:sd-exp}.
In fact, \citet{telgarsky2013margins} generalized the result to a
richer variety of exponential tailed loss functions including logistic
loss, and a broad class of non-constant step-size rules. Interestingly,  coordinate descent with exact line search can result in infinite step-sizes, leading the iterates to 
converge in a different direction that is not a max-$\ell_1$-margin
direction \citep{rudin2004dynamics}, hence the maximum step-size bound in Theorem~\ref{thm:sd-exp}.  

Theorem~\ref{thm:sd-exp} is a generalization of the result of \citeauthor{telgarsky2013margins}
 to steepest descent with respect to other norms, and our proof
follows the same strategy as \citeauthor{telgarsky2013margins}. We
first prove a generalization of the duality result of
\citet{shalev2010equivalence}: if there is a unit norm linear
separator that achieves margin $\gamma$, then $\norm{\nabla
  \c{L}(w)}_\star \ge \gamma \c{L}(w)$ for all $w$. By using this lower
bound on the dual norm of the gradient, we are able to
show that the loss decreases faster than the increase in the norm of
the iterates, establishing convergence in a margin maximizing direction.

In relating the optimization path to the regularization path, it is
also relevant to relate Theorem \ref{thm:sd-exp} to the result by \citet{rosset2004boosting} that
for monotone loss functions and  $\ell_p$ norms, the $\ell_p$ regularization path $\hat{w}(c)=\argmin_{\w:\|w\|_p\le c}
\c{L}(\W{t})$ also converges in direction to the maximum margin
separator, i.e.,~$\lim_{c\to\infty}\hat{w}(c)=\w^\star_{\|.\|_p}$ .  Although the
optimization path and regularization path are not the same, they both
converge to the same max-margin separator in the limits of $c\to \infty$ and $t\to\infty$, for the  regularization path and steepest descent optimization path, respectively.

\subsection{Adaptive Gradient Descent (AdaGrad)} \label{sec:adagrad}
Adaptive gradient methods, such as AdaGrad \citep{duchi2011adaptive} or Adam  \citep{kingma2015adam} are very popular for neural network training. We now look at the implicit bias of the  basic (diagonal) AdaGrad. 
\begin{equation}
\w_{\left(t+1\right)} =\w_{\left(t\right)}-\eta\mathbf{\h}_{\left(t\right)}^{-1/2}\nabla\mathcal{L}\left(\w_{\left(t\right)}\right),
\label{eq: AdaGrad}
\end{equation}
where $\mathbf{\h}_{\left(t\right)}\in\bR^{d\times d}$ is a diagonal matrix such that, 
\begin{equation}
\,\forall i:\,\,\mathbf{\h}_{\left(t\right)}[i,i]=\sum_{u=0}^{t}\left(\nabla\mathcal{L}\left(\w_{\left(u\right)}\right)[i]\right)^{2}\,.\label{eq: G}
\end{equation}
AdaGrad updates described above correspond to a pre-conditioned gradient descent, where the pre-conditioning matrix $\mathbf{\h}_{(t)}$ adapts across iterations. 
It was observed by \citet{wilson2017marginal}  that for neural networks with squared loss, adaptive methods tend to degrade generalization performance  in comparison to non-adaptive methods (e.g., SGD with momentum), even when both methods are used to train the network until convergence to a global minimum of training loss. This suggests that adaptivity does indeed affect the implicit bias. For squared loss, by inspection the updates in eq. \eqref{eq: AdaGrad}, we do not expect to get a characterization of the limit point $\w_\infty$  that is independent of the step-sizes. 

However, we might hope that, like for steepest descent, the situation might be different for strictly monotone losses, where the asymptotic behavior could potentially nullify the initial conditions. Examining  the updates in eq. \eqref{eq: AdaGrad}, we can see that the robustness to initialization and initial updates depend on whether the matrices $\mathbf{\h}_{(t)}$ diverge or converge: if $\mathbf{\h}_{(t)}$ diverges, then we  expect the asymptotic effects to dominate, but if it is bounded, then the limit direction will  depend  on the initial conditions.

Unfortunately, the following theorem shows that, the components of $\mathbf{\h}_{(t)}$ matrix are bounded, and hence even for strict monotone losses, the initial conditions $\wo,\mathbf{\h}_{\left(0\right)}$ and step-size $\eta$ will have a non-vanishing contribution to the asymptotic behavior of $\mathbf{\h}_{(t)}$ and hence to the limit direction $\bar{w}_{\infty}=\lim\limits_{t\to\infty}\frac{\W{t}}{\|\W{t}\|}$, whenever it exists. In other words,  the implicit bias of AdaGrad does indeed depend on initialization and step-size. 

\begin{restatable}{theorem}{thmlemadagrad}
\label{lem:AdaGrad} For any linearly separable training data $\{\x{n},\y{n}\}_{n=1}^N$,  consider the AdaGrad iterates $\W{t}$ from eq.~\eqref{eq: AdaGrad}  for minimizing $\mathcal{L}\left(\w\right)$ with exponential loss $\ell(u,y)=\exp(-uy)$. 
For any fixed and bounded step-size $\eta<\infty$, and any initialization of $\wo$ and $\mathbf{\h}_{\left(0\right)}$, such that $\frac{\eta}{2}\mathcal{L}\left(\wo\right)<1$, and $\left\Vert \mathbf{\h}_{\left(0\right)}^{-1/4}\x{n}\right\Vert _{2}\leq1$, $\forall i,\forall t:\,\,\mathbf{\h}_{\left(t\right)}[i,i]<\infty.$
\end{restatable}


\remove{
For example, if $\mathbf{\h}_{(t)}$ converges to some fixed matrix $\mathbf{\h}_{(\infty)}$, then AdaGrad asymptotically becomes GD with a fixed pre-conditioner. As this is a special case of SD, we expect $\w_{(t)}$ to converge in direction to that of the solution of the following max margin problem
\begin{equation}
\min_{\w}\w^{\top}\mathbf{\h}_{\left(\infty\right)}^{1/2}\w\quad , s.t.\,\w^{\top}\x{n}\geq1 .\label{eq: svm for Adagrad}
\end{equation}
In general though, as the matrix $\mathbf{\h}_{(\infty)}$ depends on the initial conditions, so would the minimizer of eq. \ref{eq: svm for Adagrad}. 
}

\section{Gradient  descent on the factorized parameterization}\label{sec:mf}
Consider the empirical risk minimization in eq.  \eqref{eq:lm} for matrix valued $X_n\in\bR^{d\times d}$, $W\in\bR^{d\times d}$
\begin{equation}
\min_{W} \c{L}(W)=\ell(\innerprod{W}{X_n},y_n).
\label{eq:lmW}
\end{equation}
This is the exact same setting as eq. \eqref{eq:lm} obtained by arranging $\w$ and $\x{n}$ as matrices.  We can now study  another class of algorithms for learning linear models based on matrix factorization, where we reparameterize $W$ as $W=UV^\top$ with \textit{unconstrained} $U\in\bR^{d\times d}$ and  $V\in\bR^{d\times d}$  to get the following equivalent  objective, 
\begin{equation}
\min_{U,V} \c{L}(UV^\top)=\sum_{n=1}^N\ell(\innerprod{UV^\top}{X_{n}},\y{n}).
\label{eq:lm-uv}
\end{equation}
 Note that although non-convex, eq. \eqref{eq:lm-uv} is equivalent to eq. \eqref{eq:lmW} with the exact same set of global minima over $W=UV^\top$. 
 \citet{gunasekar2017implicit} studied this problem for squared loss $\ell(u,y)=(u-y)^2$ and noted that  gradient descent on the factorization yields radically different implicit bias compared to gradient descent  on $W$. In particular, gradient descent on $U,V$ is often observed to be biased towards low nuclear norm solutions, which in turns ensures generalization \citep{srebro2005generalization} and low rank matrix recovery \citep{recht2010guaranteed,candes2009exact}. Since the matrix factorization objective in  eq. \eqref{eq:lm-uv} can be viewed as a two-layer neural network with linear activation, understanding the implicit bias here could  provide direct insights into characterizing the implicit bias in more complex neural networks with non-linear activations.

\citet{gunasekar2017implicit} noted that,  the optimization problem in eq. \eqref{eq:lm-uv} over factorization $W\!=\!UV^\top\!$ can be cast as a special case of optimization over p.s.d. matrices  with unconstrained symmetric factorization $W\!=\!UU^\top$:
\begin{equation}
\min_{U\in\bR^{d\times d}}\bar{\c{L}}(U)=\c{L}(UU^\top)=\sum_{n=1}^N\ell\left(\innerprod{UU^\top}{X_{n}},\y{n}\right).
\label{eq:lm-u}
\end{equation}
Specifically, in terms of both the objective as well as gradient descent updates, a problem instance of  eq. \eqref{eq:lm-uv} is equivalent to a problem instance of eq. \eqref{eq:lm-u} with larger data matrices $\tilde{X}_n=\left[\begin{smallmatrix}0&X_n\\X_n^\top&0\end{smallmatrix}\right]$ and loss optimized over larger p.s.d. matrix of the form  $\tilde{U}\tilde{U}^\top=\left[\begin{smallmatrix}A_1&W\\W^\top&A_2\end{smallmatrix}\right]$, where $W=UV^\top$ corresponds to the optimization variables in the original problem instance of eq. ~\eqref{eq:lm-uv} and $A_1$ and $A_2$ some p.s.d matrices that are irrelevant for the objective.

Henceforth, we will also consider the symmetric matrix factorization  in \eqref{eq:lm-u}.
Let $\U{0}\in\mathbb{R}^{d\times d}$ be any full rank initialization,  gradient descent updates in $U$ are given by,
\begin{equation}
\U{t+1}=\U{t}-\eta_t\nabla\bar{\c{L}}(\U{t}), 
\label{eq:mf-updateU}
\end{equation}
with corresponding updates in $\WW{t}=\U{t}\U{t}^\top$  given by,
\begin{flalign}
\WW{t+1}=\WW{t}\! &-\eta_t\big[\nabla{\c{L}}(\WW{t})\WW{t}+\WW{t}\nabla{\c{L}}(\WW{t})\big]+\eta_t^2\nabla{\c{L}}(\WW{t})\WW{t}\nabla{\c{L}}(\WW{t})
\label{eq:mf-updateW}
\end{flalign}

\paragraph{Losses with a unique finite root }For  squared loss, \citet{gunasekar2017implicit}  showed that the implicit bias of iterates in eq.~\eqref{eq:mf-updateW} crucially depended on both the initialization $\U{0}$ as well as the step-size $\eta$. \citeauthor{gunasekar2017implicit} conjectured, and provided theoretical and empirical evidence that
gradient descent on the factorization  converges to the minimum
nuclear norm global minimum, but only if the initialization is infinitesimally close to zero and the step-sizes
are infinitesimally small.  \citet{li2017algorithmic}, later proved the conjecture under additional assumption that the measurements $X_n$ satisfy certain \textit{restricted isometry property (RIP)}. 

In the case of  squared loss, it is evident that for finite step-sizes and finite initialization, the implicit bias towards the minimum nuclear norm global minima is not exact. In practice, not only do we need $\eta>0$, but we also cannot initialize very close to zero since zero is a saddle point for eq. \eqref{eq:lm-u}. The natural question motivated by the results in Section~\ref{sec:monotonic} is: for strictly monotone losses, can we get a characterization of the implicit bias of gradient descent for the factorized objective in eq.  \eqref{eq:lm-u} that is more robust to initialization and step-size? 

\paragraph{Strict monotone losses } In the following theorem, we again see that the characterization of the implicit bias of gradient descent for factorized objective is more robust in the case of strict monotone losses. 

\begin{restatable}{theorem}{thmmfexp} \label{thm:mf-exp}
For almost all datasets $\{X_{n},\y{n}\}_{n=1}^N$ separable by a p.s.d. linear classifier, consider the gradient descent iterates $\U{t}$ in eq. \eqref{eq:mf-updateU} for minimizing $\bar{\c{L}}(U)$ with the exponential loss $\ell(u,y)=\exp(-uy)$ and the corresponding sequence of linear predictors $\WW{t}$ in eq. \eqref{eq:mf-updateW}.  For any full rank initialization $\U{0}$ and 
any finite step-size sequence $\{\eta_t\}_t$, 
if $\WW{t}$ asymptotically minimizes $\c{L}$, i.e., $\c{L}(\WW{t})\to0$,  and additionally the   updates $\U{t}$ and the gradients $\nabla\c{L}(\WW{t})$ converge in direction, 
then the limit direction $\bar{U}_\infty=\lim\limits_{t\to\infty}\frac{\U{t}}{\|\U{t}\|_*}$ 
is a scaling of a first order stationary point (f.o.s.p) of the following non-convex optimization problem
\begin{equation}
\bar{U}_\infty\propto\text{ f.o.s.p. }\min_{U\in\bR^{d\times d}} \norm{U}^2_2 \quad \st \quad \forall n, {\y{n}\innerprod{UU^\top}{X_{n}}}\ge 1.
\label{eq:uopt}
\end{equation}
\end{restatable}
\begin{remark}Any global minimum $U^*$ of eq.~\eqref{eq:uopt} corresponds to predictor $W^*$ that minimizes the nuclear norm $\|.\|_*$ of linear p.s.d. classifier with margin constraints,
\begin{equation}W^*=\argmin_{W\succcurlyeq0}\|W\|_* \st \forall n, {\y{n}\innerprod{W}{X_{n}}}\ge1.\label{eq:wopt}\end{equation}
Additionally, in the absence of rank constraints on $U$, all second order stationary points of eq.~\eqref{eq:uopt} are global minima for the problem. More general, we expect a stronger result that $\bar{W}_\infty=\bar{U}_\infty\bar{U}_\infty^\top$, which is also the limit direction of $\WW{t}$, is a minimizer of eq. \eqref{eq:wopt}. Showing a stronger result that $\WW{t}$ indeed converges in direction to $W^*$ is of interest for future work. 
\end{remark}
Here we note that convergence of $\U{t}$ in direction is necessary for the  characterization of implicit bias  to be relevant, but in Theorem~\ref{thm:mf-exp},  we require  stronger conditions that  the  gradients $\nabla\c{L}(\WW{t})$ also converge in direction. Relaxing this condition is of interest for future work. 
\paragraph{Key property } Let us look at  exponential loss when $\WW{t}$ converges in direction to, say $\bar{W}_\infty $. Then $\bar{W}_\infty$ can be expressed as $\WW{t}=\bar{W}_\infty g(t)+\rho(t)$ for some scalar $g(t)\to\infty$ and $\frac{\rho(t)}{g(t)}\to 0$. Consequently, the gradients $\nabla\c{L}(\WW{t})=\sum_{n}\text{e}^{-g(t)\y{n}\innerprod{W_\infty}{X_{n}}}e^{-\y{n}\innerprod{\rho(t)}{X_{n}}}\,y_nX_{n}$  will asymptotically be  dominated by linear combinations of examples $X_{n}$ that have the  smallest distance to the decision boundary, i.e.,~the support vectors of $\bar W_\infty$. This  behavior  can be used to show optimality of $\bar U_\infty$ such that $\bar W_\infty=\bar U_\infty \bar U_\infty^\top$ to the first order stationary points of the maximum margin problem in eq.~\ref{eq:uopt}. 

This idea formalized in the following lemma, which is of interest beyond the results in this paper. 
\begin{restatable}{lemma}{lemgradconv} \label{lem:grad-conv}
For almost all linearly separable datasets $\{\x{n},\y{n}\}_{n=1}^N$, consider any sequence $\W{t}$ that  minimizes $\c{L}(\w)$ in eq. \eqref{eq:lm} with exponential loss, i.e.,   $\c{L}(\W{t})\to0$.
If $\W{t}$ converges in direction to a strictly separating predictor, i.e., $\bar\w_\infty=\lim\limits_{t\to\infty}\frac{\W{t}}{\|\W{t}\|}$ exists with $\min_n \y{n}\innerprod{\bar \w_\infty}{\x{n}}>0$,  then for every accumulation point $z_\infty$ of  $\Big\{\frac{-\nabla\c{L}(\W{t})}{\norm{\nabla\c{L}(\W{t})}}\Big\}_t$, 
$\exists \{\alpha_n\ge 0\}_{n\in S} \st z_\infty=\sum\limits_{n\in S}\alpha_n\y{n}\x{n},$
 where $S=\{n:\y{n}\innerprod{\bar \w_\infty}{\x{n}}=\min_n \y{n}\innerprod{\bar \w_\infty}{\x{n}}\}$ are the indices of the data points with smallest margin to $\bar \w_\infty$. 
\end{restatable}

\section{Summary}
We studied the implicit bias of different optimization algorithms for two  families of losses, losses with a unique finite root and strict monotone losses, where the biases are fundamentally different. 
In the case of losses with a unique finite root, we have a simple characterization of the limit point $w_\infty=\lim_{t\to\infty}\W{t}$ for mirror descent. But for  this family of losses, such a succinct characterization does not extend to steepest descent with respect to general norms. 
On the other hand, for  strict monotone losses, we noticed that the initial updates of the algorithm, including initialization and initial step-sizes are nullified when we analyze the asymptotic limit direction $\bar{w}_\infty=\lim\limits_{t\to\infty}\frac{\W{t}}{\|\W{t}\|}$. We show that for steepest descent, the limit direction is a maximum margin separator within the unit ball of the corresponding norm.  
We also looked at other optimization algorithms for strictly monotone losses. For matrix factorization, we again get a more robust characterization that relates the limit direction to the maximum margin separator with unit nuclear norm. This again, in contrast to squared loss  \citet{gunasekar2017implicit}, is independent of the initialization and step-size. 
However, for AdaGrad, we show that even for strict monotone losses, the limit direction $\bar{w}_\infty$ could depend on the  initial conditions.

In our results, we characterize the implicit bias for linear models as minimum norm (potential) or maximum margin solutions. These are indeed very special among all the solutions that fit the training data, and in particular, their generalization performance can in turn be understood from standard analyses  \cite{bartlett2003rademacher}.

Going forward, for more complicated non-linear models, especially  neural networks, further work is required in order to get a more complete understanding of the implicit bias. The preliminary result for matrix factorization  provides us tools to attempt extensions to multi-layer linear models, and eventually to non-linear networks.  
Even for linear models, the question of what is the implicit bias is when $\c{L}(w)$ is optimized with explicitly constraints $\w\in\c{W}$ is an open problem. We believe  similar characterizations can be obtained when there are multiple feasible solutions with $\c{L}(\w)=0$. 
We also believe, the results for single outputs considered in this paper can also be extended for multi-output loss functions.

Finally, we would like  a more fine grained analysis connecting the iterates $\W{t}$ along the optimization path of various algorithms to the regularization path, $\hat{w}(c)=\argmin_{\c{R}(w)\le c}\c{L}(w)$, where an explicit regularization is added to the optimization objective. In particular,  our positive characterizations show  that the optimization  and  regularization paths meet at the limit of $t\to\infty$ and  $c\to \infty$, respectively. It would be desirable to further understand the relations between the entire optimization and regularization paths, which will help us understand the non-asymptotic effects from early stopping.  

\section*{Acknowledgments} The authors are grateful to M.S. Nacson, Y. Carmon, and the anonymous ICML reviewers for helpful comments on the manuscript. The research was supported in part by NSF IIS award 1302662. The work of DS was supported by the Taub Foundation. 
\bibliography{suriya}

\clearpage
{
\appendix

\section{Losses with a unique finite root}
Let $\c{P}_\c{X}=\text{span}(\{\x{n}:n\in[N]\})=\{\sum_{n}\nu_n\x{n}:\nu_n\in\bR\}$. Let $\ell^\prime(u,y)$ be the derivative of $\ell$ w.r.t first operand $u$, then we can see that, for any $\ell$,
\begin{equation} 
\forall {w}\in\bR^d, \nabla\c{L}({w})=\sum_{n=1}^N\ell^\prime(\innerprod{{\w}}{\x{n}},\y{n}))\,\x{n}\in \c{P}_\c{X}.
\label{eq:grad}
\end{equation}

\subsection{Proof of Theorem~\ref{thm:md-finite}-\ref{thm:md-finite1a}} For a strongly convex potential $\psi$, denote the global optimum with minimum Bregman divergence $D_\psi(.,\wo)$ to the initialization $\wo$ as  
\begin{equation}
\w^*_\psi=\argmin_{\w} D_\psi(\w,\wo) \st \forall n, \; \innerprod{w}{\x{n}}=\y{n},
\label{eq:md-mnopt1}
\end{equation}
where recall that $D_\psi(\w,\wo)=\psi(\w)-\psi(\wo)-\innerprod{\nabla\psi(\wo)}{\w-\wo}$. 

The KKT optimality conditions for \eqref{eq:md-mnopt1} are as follows,
\begin{flalign}
\nonumber &\textbf{Stationarity: }&&\nabla\psi(\w^*_\psi)-\nabla\psi(\wo)\in\c{P}_\c{X}, \text{ or }\exists \{\nu_n\}_{n=1}^N \st \nabla\psi(\w^*_\psi)-\nabla\psi(\wo)=\sum_{n=1}^N \nu_n\x{n}&\\
&\textbf{Primal feasibility: }&&\forall n, \innerprod{\w^*_\psi}{\x{n}}=\y{n},\text{ or }\w^*_\psi\in\c{G}&
\label{eq:md-mnopt-kkt}
\end{flalign}

Recall Theorem~\ref{thm:md-finite}--\ref{thm:md-finite1a} from Section~\ref{sec:md-finite}.
{\thmmdfinite*}
{\thmlinc*}
{\thmmdfinitea*}
{\remsgd*}
\begin{proof} 
\begin{asparaenum}[(a)]
\item \textbf{Generic mirror descent: Theorem~\ref{thm:md-finite}} Recall the updates of mirror descent: 
$\nabla\psi(\W{t+1})-\nabla\psi(\W{t})=-\eta_t\nabla\c{L}(\W{t})$
Using telescoping sum, we have, 
\begin{equation}
\forall t,\;\nabla\psi(\W{t})-\nabla\psi(\W{0})=\sum_{t'<t}\nabla\psi(\W{t'+1})-\nabla\psi(\W{t'})=\sum_{t'<t}-\eta_{t'}\nabla\c{L}(\W{t'})\in\c{P}_\c{X},
\label{eq:thm1-1}
\end{equation}
where the last inclusion follows as  $\forall t^\prime, -\eta_{t'}\nabla\c{L}(\W{t'})\in\c{P}_\c{X}$ from \eqref{eq:grad}. 

Thus, for all $t$, $\W{t}$ from mirror descent updates in eq. \eqref{eq:md-upd-opt} always satisfy the stationarity condition of eq. \eqref{eq:md-mnopt-kkt}. Additionally, if $\W{t}$ converges to a global minimum, then $w_\infty=\lim_{t\to\infty} \W{t}\in \c{G}=\{\w:\forall n, \innerprod{w}{\x{n}}=\y{n}\}$ also satisfies the primal feasibility condition in eq. \eqref{eq:md-mnopt-kkt}. 
Combining the above arguments, we have that if $\c{L}(\w_\infty)=0$, then $\w_\infty=\argmin_{\w\in\c{G}} D_\psi(\w,\wo)$. 
\item \textbf{Realizable affine equality constraints: Theorem~\ref{thm:linc}} For $\c{W}=\{\w:G\w=h\}$ for some $G\in\b{R}^{d^\prime\times d}$ and $h\in\b{R}^{d^\prime}$ with a realizable feasible solution $\w\in\c{W}$ satisfying $\forall n, \innerprod{x_n}{\w}=y_n$, the claim is that $\W{t}$ from  eq. \eqref{eq:md-update} converges to 
\begin{equation}
\w^*_{\psi,\c{W}}=\argmin_{\w} D_\psi(\w,\wo) \st \forall n, \; \innerprod{w}{\x{n}}=\y{n}\text{ and } w\in\c{W}=\{w:G\w=h\}.
\label{eq:md-mnopt2}
\end{equation}

The KKT optimality conditions for \eqref{eq:md-mnopt2} are as follows, 
\begin{flalign}
\nonumber &\textbf{Stationarity: }&&\exists \{\nu_n\}_{n=1}^N, \{\mu_j\}_{j=1}^{d'} \st \nabla\psi(\w^*_{\psi,\c{W}})-\nabla\psi(\wo)=\sum_{n=1}^N \nu_n\x{n}+\sum_{j=1}^{d'} \mu_j g_j&\\
&\textbf{Primal feasibility: }&&\forall n, \innerprod{\w^*_{\psi,\c{W}}}{\x{n}}=\y{n},\text{ and }G\w^*_{\psi,\c{W}}=h&
\label{eq:md-mnopt-kkt2}
\end{flalign}

To show that the limit point of mirror descent updates in eq. \eqref{eq:md-update} satisfy the above KKT conditions, we first note that the updates are equivalently computed as follows,
\begin{equation}
\begin{split}
\W{t+1}&=\argmin_{\w:Gw=h} \eta_t \innerprod{\w}{\nabla\c{L}(\W{t})}+D_\psi(\w,\W{t})\\
&=\argmin_{\w:Gw=h} D_\psi(\w,\nabla\psi^{-1}\left(\nabla\psi(\W{t})-\eta_t\nabla\c{L}(\W{t})\right)).
\end{split}
\label{eq:md-upd2}
\end{equation}
Let $g_j$ for $j=1,2,\ldots,d'$ denote the rows of $G$. From the optimality conditions of eq. \eqref{eq:md-upd2}, we get that 
\begin{equation}
\exists \{\mu_j\}_{j=1}^{d'}, \quad\text{ s.t. }\quad
\nabla\psi(\W{t+1})=\nabla\psi(\W{t})-\eta_t\c{L}(\W{t})+\sum_{j=1}^{d'}\mu_j g_j\text{ and } G\W{t+1}=h.
\label{eq:md-upd3}
\end{equation}

Again,  primal feasibility is satisfied whenever $\W{t}\to \c{G}$ since mirror descent iterates are always feasible points $\W{t}\in\c{W}$. The stationarity condition follows from using eq. \eqref{eq:md-upd3} with same arguments of the  unconstrained case. 

\item \textbf{Dual momentum: }  For \textit{any} $\tilde{\beta}_{t'},\tilde{\gamma}_{t'}\in\bR$ and $\tilde{w}_{(t')}\in\bR^d$, consider a general update of the form 
\begin{equation}
\nabla\psi(\W{t+1})=\sum_{t'\le t} \tilde{\beta}_{t'}\nabla \psi(\W{t'})+\tilde{\gamma}_{t'} \nabla\c{L}(\tilde{\w}_{(t')}).
\label{eq:d-mom}
\end{equation}

\textbf{Claim: } If $\nabla\psi(\W{0})=0$, then for all updates of the form \eqref{eq:d-mom}  satisfies $\nabla\psi(\W{t}) \in\c{P}_\c{X}$ \\
---this can be easily proved by induction:\begin{inparaenum}\item for $t=0$, $\nabla\psi(\W{0})=0\in\c{P}_\c{X}$; \item let  $\forall t'\le t$, $\nabla \psi(\W{t'})\in\c{P}_\c{X}$, \item then using the inductive assumption and eq. \eqref{eq:grad}, we have $\nabla\psi(\W{t+1})=\sum_{t'\le t} \tilde{\beta}_{t'}\nabla \psi(\W{t'})+\tilde{\gamma}_{t'} \nabla\c{L}(\tilde{\w}_{(t')})\in\c{P}_\c{X}$.\end{inparaenum}

Dual momentum in eq. \eqref{eq:dual-mom} is a special case of eq. \eqref{eq:d-mom} with appropriate choice of $\tilde{\beta}_{t'},\tilde{\gamma}_{t'}\in\bR$ , and $\tilde{w}_{t'}\in\bR^d$. 
\item \textbf{Instancewise stochastic gradient descent (Remark~\ref{rem:sgd}): } In the above arguments, the only property of gradient $\nabla\c{L}$ that we used is that $\forall {w}\in\bR^d, \nabla\c{L}({w})=\sum_{n=1}^N\ell^\prime(\innerprod{{\w}}{\x{n}},\y{n}))\,\x{n}\in \c{P}_\c{X}$ (eq. \eqref{eq:grad}). This property also holds for instancewise stochastic gradients as defined in eq. \eqref{eq:stoc}, i.e., $\nabla\tilde{\c{L}}({w})=\sum_{n\in S_t}\ell^\prime(\innerprod{{\w}}{\x{n}},\y{n}))\,\x{n}$, hence all the results follow.  
\end{asparaenum}
\end{proof}

\subsection{Proofs of propositions in Section~\ref{sec:unique-root}}
\subsubsection{Primal momentum and natural gradient descent}
Recall the optimization problem in Examples~\ref{ex:md}--\ref{ex:ngd}: $\{(\x{1}=[1,2], \y{1}=1)\}$, and $\ell(u,y)=(u-y)^2$.  We have $\c{P}_\c{X}=\text{span}(\x{1})=\{z:2z[1]-z[2]=0\}$.

For entropy potential $\psi(w)=\sum_i\w[i]\log{\w[i]}-\w[i]$, we have $\nabla \psi(w)=\log{\w}$ (where the $\log$ is taken elementwise), and initialization 
$\W{0}=[1,1]$ satisfies $\nabla\psi(\W{0})=0$ which is the optimality condition for $\min_{\w}\psi(\w)$. 

\begin{asparaenum}
 \item  \textbf{Proof of Proposition~\ref{prop:md-primal}: } we use primal momentum with $\beta_1>0$  only in the first step, and $\forall t\ge2$, $\beta_t=\gamma_t=0$. We get the following initial updates

Since $\forall t\ge2$, $\beta_t=\gamma_t=0$, we first note that for $t>2$, the updates merely follow the path of standard MD initialized at $\nabla\psi(\W{2})$ for a convex loss function. This implies the following:
 \begin{compactitem}
\item for appropriate choice of $\{\eta_t\}_{t\ge2}$ (given by convergence analysis of mirror descent for convex functions), we can get $w_\infty=\lim_{t\to\infty}\W{t}\in\c{G}$, and
\item from eq. \eqref{eq:thm1-1}, $w_\infty$ satisfies $\nabla\psi(w_\infty)-\nabla\psi(\W{2})\in \c{P}_\c{X}\Rightarrow\nabla\psi(\w_\infty)\in\nabla\psi(\W{2})+\c{P}_\c{X}$. 
\end{compactitem}
 
 Since $w_\infty$ satisfies primal feasibility, from stationarity condition in eq. \eqref{eq:md-mnopt-kkt}, we have \[w_\infty=\w^*_\psi=\argmin_{w\in\c{G}} \psi(w) \text{\textit{ if and only if }}\nabla\psi(\W{2})\in\c{P}_\c{X}.\]
 
 We show that this is not the case for any $\beta_1>0$ and any $\gamma_1\ge0$. 
 Recall that $\Delta \W{-1}=0$, $\nabla\psi(\W{0})=0$ and $\nabla\psi(w)=\log{w}$. 
 Working through the steps in eq.  \eqref{eq:primal-mom}, for scalars $r_0=\eta_0(y_1-\innerprod{\W{0}}{x_1})$ and $\tilde{r}_1=\eta_t (y_1-\innerprod{\W{1}+\gamma_1\Delta\W{0}}{x_1})$, and any $\beta_1>0$, we have: 
 \begin{compactitem}
\item 
$\nabla\psi(\W{1})=r_0\x{1}\implies \W{1}=\exp(r_0\x{1})$, and  
\item $\nabla\psi(\W{2})=\nabla\psi((1+\beta_1)\W{1})+\tilde{r}_1\x{1}=
\log{(1+\beta_1)}+r_0\x{1}+\tilde{r}_1\x{1}\in\log{(1+\beta_1)}+\c{P}_\c{X}\notin\c{P}_\c{X}$.\mybox
\end{compactitem}
 
\item \textbf{Proof of Proposition~\ref{prop:ngd}: }  The arguments are similar to the proof of Proposition~\ref{prop:md-primal}. In Example~\ref{ex:ngd}, we again use a finite $\eta_1>0$ to get $\W{1}$ and then follow  the  NGD  using infinitesimal $\eta$ initialized at $\W{1}$. 

We know that for infinitesimal step-size, the NGD path starting at $\W{1}$ follows the corresponding infinitesimal MD path on a convex problem and hence from eq. \eqref{eq:thm1-1}, the NGD updates for this example converges to a global minimum $\w_\infty=\lim_{t\to\infty}\W{t}\in\c{G}$, that satisfies $\nabla\psi(w_\infty)-\nabla\psi(\W{1})\in \c{P}_\c{X}\Rightarrow\nabla\psi(\w_\infty)\in\nabla\psi(\W{1})+\c{P}_\c{X}$. 
 
 From stationarity condition in \eqref{eq:md-mnopt-kkt}, $w_\infty=\w^*_\psi=\argmin_{w\in\c{G}} \psi(w)$ if and only if $\nabla\psi(\W{1})\in\c{P}_\c{X}$.

For natural gradient descent, 
$\W{1}=\W{0}-\eta_1 \nabla^2\psi(\W{0})^{-1}\nabla\c{L}(\W{0})=[1+\eta_1 r_0, 1+2\eta_1 r_0]$, where $r_0=\eta_0(y_1-\innerprod{\W{0}}{x_1})$.
We then have
$\nabla\psi(\W{1})\in \c{P}_\c{X}\Leftrightarrow 2\nabla\psi(\W{1})[1]-\nabla\psi(\W{1})[2]=0\Leftrightarrow2\log{(\W{1}[1])}-\log{(\W{1}[2])}=0\Leftrightarrow\log{\big(1+\frac{\eta_1^2r^2_0}{1+2\eta_1r_0}\big)}=0.$ 

For any $\eta_1$ such that $\frac{\eta_1^2r^2_0}{1+2\eta_1r_0}\neq 0$, we get a contradiction. \mybox
\end{asparaenum}
\remove{
\subsubsection{Steepest descent: Proof of proposition~\ref{thm:sd-finite}}
Recall Example~\ref{ex:sd} with  $\{(\x{1}=[1,1,1], \y{1}=1), (\x{1}=[1,2,0], \y{1}=10)\}$, and $\ell(u,y)=(u-y)^2$. 

Consider the $\ell_p$ SD updates $\W{t}$ for this problem under \eqref{eq:sd-update}. In the continuous time limit of $\eta\to 0$, $\W{t}$ are given by:
\begin{equation}
\dv{\W{t}}{t}=\Delta\W{t}\st -\nabla{\c{L}}(\W{t})\in \partial\|\Delta\W{t}\|_p^2,
\end{equation}
\begin{equation}
\text{or equivalently  }\quad\quad\dv{\W{t}}{t}=\Delta\W{t}\st\langle\Delta\W{t},\!-\nabla\c{L}(\W{t})\rangle\!=\!\|\Delta\W{t}\|_p^2\!=\!\|\nabla\c{L}(\W{t})\|_q^2,
\label{eq:sd-ode}
\end{equation}
where $q$ is such that $\frac{1}{p}+\frac{1}{q}=1$. 

Let us define a few notations here: Let $\phi(\w)=\|w\|_p^2$, and $\nabla\psi(w)=\nabla \|w\|_p^2$. It can be verified that 
$$\nabla\psi(w)[i]=\frac{w[i]|w[i]|^{p-2}}{\|w\|_p^{p-2}},\text{ and }\nabla\psi^{-1}(z)[i]=\frac{z[i]|z[i]|^{q-2}}{\|z\|_p^{q-2}}$$

We want to show that SD w.r.t $\|.\|_p$ in the continuous time limit of satisfies 
\begin{equation}
\forall p\in(1,2), \exists \W{0}\st \w_\infty:=\lim_{t\to\infty}\W{t}\neq\argmin_{\w\in\c{G}} \|\w-\wo\|_p.
\label{eq:sd-1temp}
\end{equation}

Let us assume the contradiction of \eqref{eq:sd-1temp}, that is $\exists p\in(1,2)$ such that $\forall \W{0}$, $\W{t}$ from SD update under $\ell_p$ norm converges to $w_\infty:=\lim_{t\to\infty}\W{t}=\argmin_{\w\in\c{G}} \|\w-\wo\|_p$. 

From the optimality condition of $w_\infty=\argmin_{\w\in\c{G}} \|\w-\wo\|_p$, we require 

}
\section{Steepest descent for strictly monotone losses}
We prove Theorem \ref{thm:sd-exp} in this section. 
{\thmsdexp*}
The proof is  divided into three subsections
\begin{compactenum}
\item Generalized duality lemma: we show that  for all norms and all $w$, $\norm{\nabla \c{L}(w)}_\star \ge \gamma \c{L}(w)$.
\item  Properties of $\nabla\c{L}(\W{t})$ and $\c{L}(\W{t})$ for steepest descent: we prove two lemmata that show some useful properties of $\nabla\c{L}(\W{t})$ and $\c{L}(\W{t})$.
\item Remaining steps in the proof: putting together above lemmata to prove Theorem~\ref{thm:sd-exp}.
\end{compactenum}

\subsection{Generalized duality lemma:  for all norms and all $w$, $\norm{\nabla \c{L}(w)}_\star \ge \gamma \c{L}(w)$}
The following lemma is a standard result in convex analysis.
\begin{lemma}[Fenchel Duality]\label{lem:fenchel}
Let $A \in \b{R}^{m \times n}$, and $f:\bR^m\to\bR,g:\bR^n\to\bR$ be two closed convex functions and $f^\star, g^\star$ be their Fenchel conjugate functions, respectively. Then,
\begin{align}
\max_{w\in\bR^n} - f^\star (Aw) - g^\star (-w) \le \min_{r\in\bR^m} f(r) +g(A^\top r).
\end{align}
\end{lemma}

Let $X \in \b{R}^{N \times d}$ be the data matrix with $\x{n}$ along the rows of $X$. Without loss of generality, let $y_n=1$, as for linear models $\y{n}$ can be absorbed into $\x{n}$. Let $e_n$ denote  the $n^\text{th}$ standard basis in $\bR^N$. 

We define the $\norm{\cdot}$- maximum margin as,
\begin{equation}
\gamma =\max_{w\neq 0}\min_{n\in[N]}\frac{\innerprod{w}{\x{n}}}{\norm{w}}= \max_{\norm{w}\le1} \min_{n \in [N]} e_n ^\top Xw. 
\label{eq:sd-mm}
\end{equation}

Our primary technical novelty is the following duality lemma that generalizes similar result in  \citet{telgarsky2013margins} for $\ell_1$ norm to general norms:  
 we want to show that $\norm{\nabla \c{L}(w)}_\star \ge \gamma \c{L}(w)$ for all $w$, where  recall that $\|.\|_\star$ is the dual norm of $\|.\|$. 
 
 Define $r_n (\w) = \exp(-w^\top x_n)\ge 0$ and let $r(\w)=[r_n(\w)]_{n=1}^N\in\bR^N$. 
  Note that $\c{L}(w) = \norm{r(w)}_1$ and $\nabla \c{L}(w) =X^\top r(w) $. 
  We can now restate $\norm{\nabla \c{L}(w)}_\star \ge \gamma \c{L}(w)$   as $\frac{\norm{X^\top r(w)}_ \star }{\norm{r(w)}_1}\ge \gamma$. 
  
  In the following lemma, we show this holds for any $r_n(w) \ge0$. Since norms are homogeneous, this is equivalent to $\min_{r \in \Delta_{N-1}} \norm{X^\top r}_\star \ge \gamma,$ 
where $\Delta_{N-1}=\{v\in\bR^N:v\ge0,\|v\|_1=1\}$ is the $N$-dimensional probability simplex.

\begin{lemma} \label{lem:duality}For any norm $\|.\|$, the following duality holds:
\begin{align}
\min_{r  \in \Delta_{N-1}} \norm{ X^\top r}_\star\ge \max_{\norm{w}\le1} \min_{n \in [N]} e_n ^\top Xw= \gamma.
\end{align}
This implies, for exponential loss $\ell(u,y)=\exp(-uy)$, the following holds
\begin{equation}
\forall \w,\; \norm{\nabla \c{L}(\w)}_\star \ge \gamma \c{L}(\w).
\end{equation} 
\end{lemma}
\begin{proof} Let  $\ind{E}$ denote the indicator function which takes value $0$ if $E$ is satisfied and $\infty$ otherwise.

Define $f(r) = \ind{ r \in \Delta_{N-1}}$ and $g(z) = \norm{z}_\star$, so that
\begin{equation}
 \min_{r  \in \Delta_{N-1}} \norm{ X^\top r}_* = \min_{r\in\bR^N} f(r) +g(X^\top r).
  \label{eq:primal}
 \end{equation}
 
  The conjugates are $f^*(y) = \max_{r \in \Delta_{N-1}} \innerprod{y}{r}  = \max_{n=1}^N \innerprod{y}{e_n} $, and $g^*(w) = \ind{\norm{w}\le 1}$. The LHS of Lemma \ref{lem:fenchel} is
\begin{align}
\nonumber \max_w &\big( -f^* (Xw ) -g^* (-w) \big)=\max_w\big( - \max_n e_n ^\top X w - \ind{\norm{w} \le 1}\big)\\
 &= \max_{\norm{w} \le 1} \min_n e_n ^\top X(-w)\overset{(a)} = \max_{\norm{w} \le 1} \min_n e_n ^\top Xw \overset{(b)}= \gamma,
 \label{eq:dual}
\end{align}
where $(a)$ follows from central symmetry of $\{w:\|w\|\le1\}$, and $(b)$  from definition of maximum $\|.\|$-margin in eq.~\eqref{eq:sd-mm}. 

Using  weak duality  (Lemma~\ref{lem:fenchel}) on eqs. \eqref{eq:primal} and \eqref{eq:dual}, we have  $ \forall r,\;{\|X^\top r\|_\star}\ge \gamma {\|r\|_1}$.
Finally, recalling that for exponential loss $r_n (w) = \exp(-w^\top x_n)$, $\c{L}(w) = \norm{r(w)}_1$ and $\nabla \c{L}(w) =X^\top r(w) $, we have  $\forall w$, $\norm{\nabla \c{L}(\w)}_\star \ge \gamma \c{L}(\w). $
\end{proof}

\subsection{Properties of $\nabla\c{L}(\W{t})$ and $\c{L}(\W{t})$ for steepest descent}
Recall the steepest descent updates in eqs. \eqref{eq:sd-update} and \eqref{eq:sd-upd-opt} : 
\begin{equation}
\begin{split}
&\W{t+1}= \W{t}+\eta_t \Delta \W{t},\text{ where }\Delta \W{t}\text{ satisfies }  \\
&\langle\Delta\W{t},-\nabla\c{L}(\W{t})\rangle=\|\Delta\W{t}\|^2=\|\nabla\c{L}(\W{t})\|_\star^2. 
\end{split}
\label{eq:sd-update-1}
\end{equation}

\begin{lemma} \label{lem:sum}For exponential loss $\ell(u,y)=\exp(-uy)$, consider the steepest descent iterates $\W{t}$ for minimizing $\c{L}(\W{t})$, with  any initialization $\W{0}$ and any finite step-size $\eta_t$ that leads to a strictly decreasing sequence  $\c{L}(\W{t})$ and satisfies $0<\eta_t \le \min\{\eta_+, \frac{1}{B^2\c{L}(\W{t})}\}$, where $B=\max_n\|\x{n}\|_\star$. Then the following holds:
\begin{compactenum}[(A)]
\item $\sum_{t=0}^\infty\eta_t \|\nabla\c{L}(\W{t})\|_\star^2\le \infty$, and hence $\|\nabla\c{L}(\W{t})\|_\star\to 0$.
\item Iterates $\W{t}$ converge to a global minima $\c{L}(\W{t})\to 0$, and hence $\forall n$ $\innerprod{\W{t}}{\x{n}}\to\infty$.
\item $\sum_{t=0}^\infty\eta_t\|\nabla\c{L}(\W{t})\|_\star= \infty. $
\end{compactenum}
\end{lemma}
\begin{proof}
\begin{asparaenum}
\item \textbf{Proof of $(A)$: }
We have that $\norm{x_n}_\star\le B$ for all $n$. Recall that $r_n(w)=\exp(-\innerprod{w}{\x{n}})\ge 0$, $\c{L}(w)=\sum_nr_n(w)$,  and $\nabla\c{L}(w)=\sum_{n}r_n(w) \x{n}$. Thus, for all $v$, we have \begin{equation}
v^\top\nabla ^2 \c{L} (w) v = \sum_{n} r_n(w) (x_n ^\top v)^2 \le \sum_n r_n(w) \norm{x_n}_\star ^2 \norm{v}^2 \le \c{L}(w) B^2 \norm{v}^2.
\label{eq:hessian}
\end{equation}
Using Taylor's reminder theorem for the convex loss $\c{L}$, we have 
\begin{equation}
\begin{split}
\mathcal{L}(\W{t+1}) &\leq \mathcal{L}(\W{t})+\eta_t 
\innerprod{\nabla\mathcal{L}(\W{t})}{\Delta\W{t}}+\sup_{\beta\in(0,1)}\frac{\eta_t^2}{2}{\Delta\W{t}}^\top\nabla ^2 \c{L} \left(\W{t}+\beta\eta_t\Delta\W{t}\right) \Delta\W{t}\\
&\overset{(a)} \leq\mathcal{L}(\W{t})-\eta_t\left\Vert \nabla\mathcal{L}(\W{t})\right\Vert ^{2}_\star+\frac{\eta_t^2B^2}{2}\sup_{\beta\in(0,1)}\c{L}\left(\W{t}+\beta\eta_t\Delta\W{t}\right)\left\Vert {\Delta\W{t}}\right \Vert ^{2}\\
 &\overset{(b)}{\leq}\mathcal{L}(\W{t})-\eta_t\left\Vert \nabla\mathcal{L}(\W{t})\right\Vert ^{2}_\star+\frac{\eta_t^{2}B^2}{2}\c{L}(\W{t})\left\Vert \Delta\W{t}\right\Vert ^{2} \\
 &\overset{(c)}\le\mathcal{L}(\W{t})-\frac{\eta_t}{2}\left\Vert \nabla\mathcal{L}(\W{t})\right\Vert ^{2}_\star,
\end{split}
\label{eq:lw}
\end{equation}
where $(a)$ follows from eq.~\eqref{eq:hessian} and from the condition on update direction in eq.  \eqref{eq:sd-update-1}; $(b)$ follows as $\eta_t\Delta\W{t}$ is a descent step and along with convexity of $\c{L}(\w)$ we have $\sup_{\beta\in(0,1)}\c{L}\left(\W{t}+\beta\eta_t\Delta\W{t}\right)\le \c{L}(\W{t})$; and $(c)$ follows as $\eta_t\le \frac{1}{B^2\c{L}(\W{t})}$ from the assumption  and also using $\|\Delta\W{t}\|=\|\nabla\c{L}(\W{t})\|_\star$ from eq.~\ref{eq:sd-update-1}. 

Thus, ${\mathcal{L}(\W{t})-\mathcal{L}(\W{t+1})}\geq \frac{\eta_t}{2}\left\Vert \nabla\mathcal{L}(\W{t})\right\Vert ^{2}_\star$, 
which implies 
\begin{equation}
\forall t, \; \sum_{u=0}^{t}\eta_u\left\Vert \nabla\mathcal{L}(\W{u})\right\Vert ^{2}_\star \leq2 \sum_{u=0}^{t}{\mathcal{L}(\W{u})-\mathcal{L}(\W{u+1})}=2\left(\mathcal{L}(\W{0})-\mathcal{L}(\W{t+1})\right)<\infty\,.
\end{equation}
where the final inequality follows as  $\c{L}(\W{0})<\infty$ and $\mathcal{L}(\W{t})\ge0$ $\forall t$.

In the continuous time limit of $\eta\to0$, $(A)$ is equivalently expressed as $\int_{0}^t\|\nabla\c{L}(\W{t})\|_\star^2<\infty$. Thus, we have  $\lim_{t \to \infty} \|\nabla\c{L}(\W{t})\|_\star =0$---both for any finite $\eta_t>0$ as well as in the continuous time limit of $\eta\to0$. 

\item \textbf{Proof of $(B)$ and $(C)$ :} Consider any $v\in\bR^d$ that linearly separates the data, i.e., $\forall n, \innerprod{v}{x_n} >0$ (such a $v$ always exists for linearly separable data), then we have
\[
\forall t<\infty,\;v^\top \nabla \c{L} ( w_{(t)}) = \sum_{n \in [N]} \exp( -\innerprod{\W{t}}{\x{n}} ) x_n ^\top v >0.
\]

Since $\lim_{t\to \infty}  v^\top \nabla \c{L} ( \W{t}) =0$, it must be that $\forall n,\;\exp( - \innerprod{\W{t}}{x_n})\to 0$,  and thus $\norm{\W{t}} \to \infty.$

Using triangle inequality, we have 
\begin{align}
\infty=\lim_{t\to\infty }\norm{w_{(t)} } \le \|\W{0}\|+\sum_{t=0}^\infty \eta_t\|\Delta\W{t}\|=\|\W{0}\|+\sum_{t=0}^\infty\eta_t \|\nabla\c{L}(\W{t})\|_\star,
\end{align}
where we used $\|\Delta\W{t}\|=\|\nabla\c{L}(\W{t})\|_\star$ from \eqref{eq:sd-update-1}. This gives us $\sum_{t=0}^\infty\eta_t \|\nabla\c{L}(\W{t})\|_\star =\infty$ in $(C)$.\qedhere
\end{asparaenum}
\end{proof}

We next show that under the conditions of Theorem~\ref{thm:sd-exp}, $\cL(w_{(t)})$ forms a decreasing sequence, and hence satisfies the assumption in Lemma~\ref{lem:sum}. 
\begin{lemma}\label{lem:dec} If step-sizes $\eta_t$ satisfy $ \eta_t  = \frac{c_t}{B^2\cL(w_t)}$ for $c_t \le \sqrt{2}$ , then $\cL(w_{(t+1)}) \le \cL(w_{(t)})$.
\end{lemma}
\begin{proof}

From the Taylor expansion of $\mathcal{L}\left(\w\right)$ in eq. \eqref{eq:lw}, we have 
\begin{equation}
\begin{split}
\mathcal{L}(\w_{(t+1)})  &\leq \mathcal{L}(\w_{(t)})-\eta_t\|\nabla\c{L}(\W{t})\|_\star^2+\frac{\eta_t ^2B^2}{2} \|\nabla\c{L}(\W{t})\|_\star^2 \sup_{\beta\in(0,1)}\mathcal{L}(\W{t}+\beta\eta_t\Delta\W{t})\\
&\overset{(a)} \le \mathcal{L}(\W{t})-\eta_t\|\nabla\c{L}(\W{t})\|_\star^2\left(1-\frac{\eta_t B^2}{2}\max{(\mathcal{L}(\W{t}),\mathcal{L}(\W{t+1}))}\right),
\end{split}
\label{eq: L inequality1}
\end{equation}
where $(a)$ follows from convexity of $\mathcal{L}$.
	
We want to show that $\cL(w_{(t+1)}) \le \cL(w_{(t)})$. Let us assume the contrary that $\cL(w_{(t+1)})> \cL(w_{(t)})$. 

From eq. \eqref{eq: L inequality1}, we have 
	\begin{align}
	\nonumber\cL(w_{(t)})\overset{(a)}<\cL(w_{(t+1)})&\le \cL(w_{(t)}) -\eta_t \|\nabla\c{L}(\W{t})\|_\star^2\left( 1- \frac{ \eta_tB^2\cL(w_{t+1})}{2}\right)\\
	&\overset{(b)}\implies \left( 1- \frac{ \eta_tB^2\cL(w_{t+1})}{2}\right)\le0,
	\label{eq:ldec}
	\end{align}
	where $(a)$ follows from the contradictory assumption and $(b)$ follows as $\eta_t \|\nabla\c{L}(\W{t})\|_\star^2\ge0$.
	
Following up from eq. \eqref{eq:ldec}, we have 
	\begin{align*}
	\cL(w_{(t+1)})&\le \cL(w_{(t)}) + \eta_t \|\nabla\c{L}(\W{t})\|_\star^2\left( \frac{ \eta_tB^2\cL(w_{t+1})}{2}-1\right)\\
	&\overset{(a)}\le  \cL(w_{(t)}) + \eta_t B^2\c{L}(\W{t})^2\left( \frac{ \eta_tB^2\cL(w_{t+1})}{2}-1\right)\\
	&\overset{(b)}\le\cL(w_{(t)}) -c_t \cL(w_t) +\frac{ c_t^2\cL(w_{t+1})}{2}\\
	\implies \cL(w_{(t+1)})&\le \frac{1-c_t}{1-0.5 c_t^2}\cL(w_{(t)})\overset{(c)}\le \cL(w_{(t)}).
	\label{eq:ldec2}
	\end{align*}
	where in $(a)$ we used $\norm{\nabla \cL(w_t) }_\star = \norm{\sum_n \exp(w_t^T x_n) x_n }_\star \le B \cL(w_t)$ from triangle inequality and $\left( \frac{ \eta_tB^2\cL(w_{t+1})}{2}-1\right)\ge0$ from eq. \eqref{eq:ldec}, $(b)$ follows from using $\eta_t  \le \frac{c_t}{B^2\cL(w_t)}$ for some $0<c_t\le \sqrt{2}$, and $(c)$ follows as for  $0<c_t\le\sqrt{2}$, $\frac{1-c_t}{1-0.5 c_t^2}\le1$. 
	 This shows $\cL(w_{t+1}) \le \cL(w_t)$ which is a contradiction. 
\end{proof}

\subsection{Remaining steps in the proof of Theorem~\ref{thm:sd-exp}}
The steepest descent updates in eq. \eqref{eq:sd-update-1} can be equivalently written as:
\begin{equation}
\begin{split}
w_{(t+1)} &= w_{(t)} - \eta_t \gamma_t p_{(t)},\text{ where } \\
\gamma_t &\triangleq \norm{\nabla \c{L}(w_{(t)})}_\star,\text{ and } p_{(t)}\triangleq\frac{\Delta\W{t}}{\norm{\nabla \c{L}(w_{(t)})}_\star},\text{ which satisfies} \\
\innerprod{p_{(t)}}{\nabla \c{L}(w_{(t)}) }&= \norm{\nabla \c{L}(w_{(t)}) }_\star \text{,   }\norm{p_{(t)}} =1.
\end{split}
\label{eq:sd-update-2}
\end{equation}

From eq .~\ref{eq:lw}, using $\gamma_t=\|\nabla\c{L}(\W{t})\|_\star=\|\Delta\W{t}\|$, we have that  
\begin{equation}
\begin{split}
\c{L}(\W{t+1})&\le\c{L}(\W{t})-{\eta_t}{\gamma_t^2}+\frac{\eta_t^{2}B^2\c{L}(\W{t})\gamma_t^2}{2}=\c{L}(\W{t})\left[1-\frac{\eta_t\gamma_t^2}{\c{L}(\W{t})}+\frac{\eta_t^{2}B^2\gamma_t^2}{2}\right]\\
&\overset{(a)}\le \c{L}(\W{t})\exp\left(-\frac{\eta_t\gamma_t^2}{\c{L}(\W{t})}+\frac{\eta_t^{2}B^2\gamma_t^2}{2}\right)\\
&\overset{(b)}\le \c{L}(\W{0})\exp\left(-\sum_{u\le t} \frac{\eta_u\gamma_u^2}{\c{L}(\W{u})}+\sum_{u\le t}\frac{\eta_u^2B^2\gamma_u^2}{2}\right), 
\end{split}
\label{eq:ellbound}
\end{equation}
where we get $(a)$ by using $(1+x)\le \exp(x)$, and $(b)$ using recursion. 

\textbf{Step 1: Lower bound the unnormalized margin: } From eq.~\eqref{eq:ellbound}, we have, 
\begin{align}
\max_{n \in[N]} \exp(-\innerprod{\W{t+1}}{x_n})& \le \c{L}(w_{(t+1)})\le 
\c{L}(\W{0})\exp\left(-\sum_{u\le t} \frac{\eta_u\gamma_u^2}{\c{L}(\W{u})}+\sum_{u\le t}\frac{\eta_u^2B^2\gamma_u^2}{2}\right).
\end{align}

By  applying $-\log$,
\begin{align}
\min_{n \in[N]} \innerprod{\W{t+1}}{x_n} \ge \sum_{u\le t}\frac{\eta_u\gamma_u ^2}{\c{L}(w_{(u)})} -\sum_{u\le t}\frac{\eta_u^2B^2\gamma_u^2}{2}- \log \c{L}(\W{0}).
\label{eq:LB}
\end{align}
\textbf{Step 2: Upper bound $\boldsymbol{\norm{w_{(t+1)}}}$: } Using $\|\Delta\W{u}\|=\|\nabla\c{L}(\W{u})\|_\star=\gamma_u$, we  have,
\begin{align}
\norm{w_{(t+1)}} \le \norm{\W{0}}+\sum_{u\le t } \eta_u \norm{\Delta\W{u}} \le \|\W{0}\|+\sum_{u\le t} \eta_u \gamma_u.
\label{eq:UB}
\end{align}
\textbf{Step 3: Lower bound on normalized margin: } Combining eqs. \eqref{eq:LB} and \eqref{eq:UB}$\forall n \in [N]$, we have that
\begin{align}
\frac{\innerprod{\W{t+1}}{x_n}}{\norm{\W{t+1}}} &\ge \frac{ \sum_{u\le t} \frac{\eta_u\gamma_u^2}{\c{L}(\W{u})}}{\sum_{u\le t}\eta_u \gamma_u +\|\W{0}\|} -\left(\frac{\sum_{u\le t}\frac{\eta_u^2B^2\gamma_u^2}{2}+\log \c{L}(\W{0})}{\|\W{t+1}\|}\right).\\
& := (I)+(II). 
\label{eq:final}
\end{align}
We look at the two terms separately, 
\begin{asparaenum}[(I)]
\item From the duality Lemma~\ref{lem:duality}, we have $\gamma_u=\|\nabla\c{L}(\W{u})\|_\star \ge \gamma \c{L}(\W{u})$. Hence,  $\sum_{u\le t} \frac{\eta_u\gamma_u^2}{\c{L}(\W{u})}\ge \gamma\sum_{u\le t}\eta_u\gamma_u$ and further using $\sum_{u\le t}\eta_u\gamma_u\to\infty$ from Lemma~\ref{lem:sum}, we have
\[\frac{ \sum_{u\le t} \frac{\eta_u\gamma_u^2}{\c{L}(\W{u})}}{\sum_{u\le t}\eta_u \gamma_u +\|\W{0}\|} \ge \gamma\frac{\sum_{u\le t}\eta_u\gamma_u}{\sum_{u\le t}\eta_u\gamma_u+\|\W{0}\|}\to \gamma\]
\item For any bounded $\eta\le \eta_+$, ${\sum_{u\le t}\frac{\eta_u^2B^2\gamma_u^2}{2}}\le\frac{\eta_+B^2}{2}\sum_{u\le t}\eta_u\gamma_u^2<\infty$ (from Lemma~\ref{lem:sum}). \\
Along with using $\|\W{t}\|\to\infty$ from Lemma~\ref{lem:sum}, we get  $\frac{\sum_{u\le t}\frac{\eta_u^2B^2\gamma_u^2}{2}+\log \c{L}(\W{0})}{\|\W{t+1}\|}\to 0$.
\end{asparaenum}

Using the above bounds in \eqref{eq:final}, we get  $\lim_{t\to \infty} \frac{w_{(t+1)}^\top x_n}{\norm{w_{(t+1)}} }\ge \gamma:=\max_{w}\frac{\w^\top x_n}{\norm{\w} }$
\mybox

\section{Adagrad}
\begin{lemma}
\label{lem: SD convergence}
Let $\mathcal{L}\left(\w\right)=\sum_{n=1}^{N}\exp\left(-\w^{\top}\x{n}\right)$,
$\left\Vert \cdot\right\Vert _{t}$ be some $\w_{\left(t\right)}$-dependent
norm, and $\left\Vert \cdot\right\Vert _{t,*}$ be its dual, and assume
that and $\forall t:\,\left\Vert \x{n}\right\Vert _{t,*}\leq1$.
We examine the following adaptive steepest descent update sequence w.r.t adaptive norm $\left\Vert \cdot\right\Vert _{t}$: 
\begin{equation}
\w_{\left(t+1\right)} =\w_{\left(t\right)}-\eta\gamma_{t}\p_{\left(t\right)},\label{eq: SD}
\end{equation}
where $\left\Vert \nabla\mathcal{L}\left(\w_{\left(t\right)}\right)\right\Vert _{t,*}\triangleq\gamma_{t}$ and $\p_{\left(t\right)}$ is the normalized update  satisfying $\left\Vert \p_{\left(t\right)}\right\Vert _{t}=1$ and $
\p_{\left(t\right)}^{\top}\nabla\mathcal{L}\left(\w_{\left(t\right)}\right)=\left\Vert \nabla\mathcal{L}\left(\w_{\left(t\right)}\right)\right\Vert _{t,*}$. 

For these adaptive steepest descent updates, for any initialization, $\wo$ such that $\eta\mathcal{L}\left(\wo\right)<1$, if $\W{t}$ minimizes $\c{L}$, i.e., $\c{L}(\W{t})\to0$, then 
we have  $\sum_{u=0}^{\infty}\gamma_{t}^{2}<\infty$.
\end{lemma}
\begin{proof}
First we note that since $\left\Vert \x{n}\right\Vert _{t,*}\leq1$ and $\left\Vert \p_{\left(t\right)}\right\Vert _{t}=1$.
\begin{equation}\label{eq:ada1}
\p_{\left(t\right)}^{\top}\nabla^{2}\mathcal{L}\left(\w\right)\p_{\left(t\right)}=\sum_{n=1}^{N}\exp\left(-\w^{\top}\x{n}\right)\left(\x{n}^{\top}\p_{\left(t\right)}\right)^{2}\leq\sum_{n=1}^{N}\exp\left(-\w^{\top}\x{n}\right)=\mathcal{L}\left(\w\right).
\end{equation}
Additionally, following the arguments of Lemma~\ref{lem:dec}, we can show that $-\eta\gamma_{t}\p_{\left(t\right)}$  for $\eta\le \frac{1}{\c{L}(\W{0})}$ is a descent direction, hence from convexity of $\c{L}$, we have
\begin{equation}\label{eq:ada2}
\max_{r\in\left(0,1\right)}\p_{\left(t\right)}^{\top}\nabla^{2}\mathcal{L}\left(\w_{\left(t\right)}-r\eta\gamma_{t}\p_{\left(t\right)}\right)\p_{\left(t\right)}\leq\max_{r\in\left(0,1\right)}\mathcal{L}\left(\w_{\left(t\right)}-r\eta\gamma_{t}\p_{\left(t\right)}\right)\leq\mathcal{L}\left(\w_{\left(t\right)}\right).
\end{equation}

From the Taylor expansion of $\mathcal{L}\left(\w\right)$ 
\begin{equation}
\mathcal{L}\left(\w_{\left(t+1\right)}\right)  \leq\mathcal{L}\left(\w_{\left(t\right)}\right)-\eta\gamma_{t}\nabla\mathcal{L}\left(\w_{\left(t\right)}\right)^{\top}\p_{\left(t\right)}+\frac{1}{2}\eta^2\gamma_{t}^{2}\max_{r\in\left(0,1\right)}\p_{\left(t\right)}^{\top}\nabla^{2}\mathcal{L}\left(\w_{\left(t\right)}-r\eta\gamma_{t}\p_{\left(t\right)}\right)\p_{\left(t\right)}.\,
\label{eq: L inequality}
\end{equation}
Substituting eq. \eqref{eq:ada1} and \eqref{eq:ada2} into eq. \ref{eq: L inequality}, we find
\begin{align*}
\mathcal{L}\left(\w_{\left(t+1\right)}\right) & \leq\mathcal{L}\left(\w_{\left(t\right)}\right)-\eta\gamma_{t}\nabla\mathcal{L}\left(\w_{\left(t\right)}\right)^{\top}\p_{\left(t\right)}+\frac{1}{2}\eta^2\gamma_{t}^{2}\mathcal{L}\left(\w_{\left(t\right)}\right)\\
 & =\mathcal{L}\left(\w_{\left(t\right)}\right)-\eta\left(1-\frac{\eta}{2}\mathcal{L}\left(\w_{\left(t\right)}\right)\right)\gamma_{t}^{2}
 \overset{(a)}\le \mathcal{L}\left(\w_{(t))}\right)-\frac{\eta}{2}\gamma_{t}^{2},
\end{align*}
where $(a)$ follows from assumption that $\eta\le\frac{1}{\c{L}(\W{0})}\le \frac{1}{\c{L}(\W{t})}$. 

Summing over the last equation, we get  that $
\frac{\eta}{2}\sum_{u=1}^{t}\gamma_{u}^{2}\le \mathcal{L}\left(\w_{\left(0\right)}\right)-\mathcal{L}\left(\w_{\left(t\right)}\right)<\infty.
$
\end{proof}

Recall the AdaGrad update
$\w_{\left(t+1\right)} =\w_{\left(t\right)}-\eta\mathbf{\h}_{\left(t\right)}^{-1/2}\nabla\mathcal{L}\left(\w_{\left(t\right)}\right)$, 
where $\mathbf{\h}_{\left(t\right)}$ is a diagonal matrix such that
\[
\,\forall i:\,\,\mathbf{\h}_{\left(t\right)}[i,i]=\sum_{u=0}^{t}\left(\nabla\mathcal{L}\left(\w_{\left(u\right)}\right)[i]\right)^{2}\,.
\]

We now prove the Theorem~\ref{lem:AdaGrad}. Recall the statement,
{\thmlemadagrad*}

\begin{proof}
First, we note that AdaGrad is a special case of the adaptive steepest descent
algorithm described in Lemma \ref{lem: SD convergence} with respect the
norm $\left\Vert v \right\Vert _{t}=\left\Vert \mathbf{\h}_{\left(t\right)}^{1/2}v \right\Vert _{2}$.
Here the dual norm $\left\Vert v \right\Vert _{t,*}=\left\Vert \mathbf{\h}_{\left(t\right)}^{-1/2}v \right\Vert _{2}$.

Also from the definition of $\mathbf{\h}_{\left(t\right)}$, we have that  $\mathbf{\h}_{\left(t\right)}^{-1}[i,i]$ is monotonically decreasing for all
$t$, and thus $\left\Vert \mathbf{\h}_{\left(t\right)}^{-1/2}\x{n}\right\Vert _{2}\leq\left\Vert \mathbf{\h}_{\left(0\right)}^{-1/2}\x{n}\right\Vert _{2}\leq1$,
and so we can apply Lemma \ref{lem: SD convergence}. This implies
that 
\begin{align*}
\infty & >\sum_{t=0}^{\infty}\left\Vert\nabla\c{L}(\W{t})\right\Vert_{t,*}^2=\sum_{t=0}^{\infty}\left\Vert \mathbf{\h}_{\left(t\right)}^{-1/2}\nabla\mathcal{L}\left(\w_{\left(t\right)}\right)\right\Vert _{2}^{2}\\
 & =\sum_{i=1}^{d}\sum_{t=0}^{\infty}\left(\nabla\mathcal{L}\left(\w_{\left(t\right)}\right)[i]\right)^{2}\left[\sum_{u=0}^{t}\left(\nabla\mathcal{L}\left(\w_{\left(u\right)}\right)[i]\right)^{2}\right]^{-1/2}\\
 & \geq\sum_{i=1}^{d}\sum_{t=0}^{\infty}\left(\nabla\mathcal{L}\left(\w_{\left(t\right)}\right)[i]\right)^{2}\left[\sum_{u=0}^{\infty}\left(\nabla\mathcal{L}\left(\w_{\left(u\right)}\right)[i]\right)^{2}\right]^{-1/2}\\
 & =\sum_{i=1}^{d}\sqrt{\sum_{t=0}^{\infty}\left(\nabla\mathcal{L}\left(\w_{\left(t\right)}\right)[i]\right)^{2}}
\end{align*}
This implies that 
\[
\forall i,\forall t:\;\mathbf{\h}_{\left(t\right)}[i,i]=\sum_{u=0}^{t}\left(\nabla\mathcal{L}\left(\w_{\left(u\right)}\right)[i]\right)^{2}\le\sum_{t=0}^{\infty}\left(\nabla\mathcal{L}\left(\w_{\left(t\right)}\right)[i]\right)^{2}<\infty\,,
\]
\end{proof}

\section{Gradient descent on factorized parameterization} 
\input{mf-proof}

\section{Preliminaries}
\begin{lemma}[Sub-differentials of norms] For a generic norm $\|v\|$ for $v\in\c{V}$, recall the dual norm $\|y\|_\star =\sup_{\|v\|\le 1} \innerprod{y}{v}$. The sub-differential of a norm $\|.\|$ at $v$ is defined as $\partial\|v\|=\{y:\forall \Delta\in\c{V},\; \|v+\Delta\|\ge \|v\|+\innerprod{y}{\Delta}\}$.

We have the following results on the properties on the sub-differentials are readily established:
\begin{asparaenum}
\item $\partial\|v\|=\{y:\|y\|_\star=1 \text{, and } \innerprod{y}{v}=\|v\|\}$
\item $y\in\partial \|v\|^2$ if and only if $v\in\partial \|v\|^2$
\item if there exists $v_1,v_2\in\c{V}$ and $g\in\c{V}^\star$ such that $g\in\partial \|v_1\|$ and $g\in\partial \|v_2\|$, then $forall \alpha,\beta>0$, $g\in\partial \|\alpha v_1+\beta v_2\|$.
\end{asparaenum}
\end{lemma}
\begin{proof}
\begin{asparaenum}
\item It can be easily verified that $\{y:\|y\|_\star=1 \text{, and } \innerprod{y}{v}=\|v\|\}\subseteq \partial \|v\|$
Conversely,  $\forall y\in\partial\|v\|$, from the definition, we have $\forall \Delta$, $\|v\|+\|\Delta\|\ge\|v+\Delta\|\ge \|v\|+\innerprod{y}{\Delta} \implies \|y\|_\star=\sup_{\Delta\neq 0} \langle\frac{\Delta}{\|\Delta\|},y\rangle \le 1$.  
Using $\|y\|_\star\le 1$ along with $\Delta=-v$, we have $\innerprod{y}{v}\ge \|v\|=\sup_{\|y\|_\star\le 1}\innerprod{v}{y}\Rightarrow \innerprod{y}{v}=\sup_{\|y\|_\star\le 1}\innerprod{v}{y} \|v\|$, which by homogeneity of norms implies $\|y\|_\star=1$. 
\item From above result, $y\in\partial \frac{1}{2}\|v\|^2\Leftrightarrow \|y\|_*=\|v\|\text{ and }{\innerprod{y}{v}}=\|v\|^2 = \|y\|^2_\star\Leftrightarrow v\in\partial \|y\|_\star^2$. 
\item 
$g\in\partial \|v_1\|\cap \partial \|v_2\|$ implies $\|g\|_\star = 1$,  $\|v_1\|=\innerprod{g}{v_1}$,  and  $\|v_2\|=\innerprod{g}{v_2}$. Using triangle inequality, 
$\|\alpha v_1+\beta v_2\|\le \alpha \|v_1\|+\beta\|v_2\|=\innerprod{g}{\alpha v_1+\beta v_2}\le \sup_{\|y\|_\star\le 1}\innerprod{y}{\alpha v_1+\beta v_2}=\|\alpha v_1+\beta v_2\|\implies \|\alpha v_1+\beta v_2\|=\innerprod{g}{\alpha v_1+\beta v_2}$. 
\end{asparaenum}
\end{proof}

\begin{lemma} [Limit points of a compact sets] If $\{a_t\}_{t=1}^\infty$ is a sequence contained in a compact set $a_t\in C$, then there exists at least one limit point of $\{a_t\}$ in $C$. That is, $\exists a^\infty\in C$ and  a subsequence $\{a_{t_k}\}_{k=1}^\infty$, such that $\lim_{k\to\infty} a_{t_k}=a^\infty$. 
\end{lemma}
\begin{theorem}[L-Hopital's Rule, proof in Theorem~$30.2$ of \citet{ross1980elementary}] Let $s\in\bR\cup \{-\infty,\infty\}$,  and $f(x)$ and $g(x)$ be continuous and differentiable functions such that $\lim_{x\to s} \frac{f^\prime(x)}{g^\prime(x)}=L$ exists. If either $(a)$ $\lim_{x\to s}f(x)=\lim_{x\to s}g(x)=0$, or $(b)$ $\lim_{x\to s}|g(x)|=\infty$, then $\lim_{x\to s} \frac{f(x)}{g(x)}$ exists and is equal to $L$.
\label{thm:lhopital}
\end{theorem}

\begin{theorem}[Stolz--Cesaro theorem, proof in Theorem~$1.22$ of \citet{muresan2009concrete}] Assume that $\{a_k\}_{k=1}^{\infty}$  and $\{b_k\}_{k=1}^{\infty}$ are two sequences of real numbers such that $\{b_k\}_{k=1}^{\infty}$ is strictly monotonic and diverging (i.e., monotonic increasing with $b_k\to\infty$ or monotonic decreasing with $b_k\to-\infty$).
Additionally, if $\lim_{k\to\infty} \frac{a_{k+1}-a_k}{b_{k+1}-b_k}=L$ exists, then $\lim_{k\to\infty} \frac{a_{k}}{b_{k}}$ exists and is equal to $L$. 
\label{thm:stolzcesaro}
\end{theorem}
}
\end{document}

%% file: mf-proof.tex
We first prove the Lemma~\ref{lem:grad-conv} on convergence of $-\nabla\c{L}(\W{t})$. This lemma holds for any general linear model \eqref{eq:lm} with exponential loss $\ell(u,y)=\exp(-uy)$ and are of interest beyond the matrix factorization setup in Section~\ref{sec:mf}.  

\subsection{Convergence of $-\nabla\c{L}(\W{t})$}
{\lemgradconv*}

Here \textit{for almost all $\{\x{n},y_n\}$} means that with probability $1$ over the dataset $\{\x{n},y_n\}$ such that the signed features $y_n\x{n}$ are  drawn independently from a distribution that is absolutely continuous w.r.t the $d$ dimensional Lebesgue measure. 
\begin{proof} Without loss of generality assume $\forall n, y_n=1$, else the sign of $y$ can be absorbed into $x$ as  $\x{n}\gets y_nx_n$.

 Let $X\in\bR^{N\times d}$ denote the data matrix with $\x{n}\in\bR^d$ along the rows of $X$. Also, for any $J\subseteq [N]$, $X_J\in\bR^{|J|\times d}$ denotes the  submatrix of $X$ with only the rows corresponding to indices in $J$.

We have that $\lim_{t\to\infty}\c{L}(\W{t})=0$ for strictly monotone loss over separable data, this implies asymptotically $\W{t}$ satisfies $X\W{t}>0,\|\W{t}\|\to\infty$. 

Since $\W{t}$ converges in direction to $\bar w_\infty$, we can write $\W{t} = g(t) \bar \w_\infty +\rho_{(t)} $ for a scalar $g(t)=\|\W{t}\|\to\infty$ and  vector $\rho(t)\in\bR^d$ such that $\frac{\rho(t)}{g(t)}\to0$. Additionally, this implies $\forall n$, $X\bar \w_{\infty}>0$. 

We introduce some additional notation:
\begin{compactitem}
\item  Denote the asymptotic  margin of $x_n$ as $\bar{\gamma}_n:=\innerprod{\x{n}}{\bar{\w}_\infty}$. Additionally, we define the following:
\begin{compactitem}
\item  
Let $\gamma=\min_n\innerprod{\x{n}}{\bar \w_\infty}=\min_n e_n^\top X\bar{\w}_\infty>0$ denote the smallest margin, where $e_n\in\bR^N$ are standard basis.  
\item Let $S:=\{n:\innerprod{\x{n}}{\bar{\w}_\infty}=\gamma\}$ denote the indices of support vectors of $\bar{\w}_\infty$.
\item Denote the second smallest margin of $\bar{\w}_\infty$ as $\bar{\gamma}:=\min_{n\notin S}\innerprod{\x{n}}{\bar{\w}_\infty}>\gamma$.
\end{compactitem}
\item  Define $\alpha_n(t) := \exp(-\innerprod{\rho(t)}{x_n})$ and let $\alpha(t)\in\bR^N$ be a vector of $\alpha_n(t)$ stacked.   For any $J\subset[N]$ and $\alpha\in\bR^N$, similar to $X_J$, let $\alpha_J\in\bR^{|J|}$ be the sub-vector with components corresponding to the indices in $J$
\item $B= \max_n \norm{x_n}_2$,  
\end{compactitem}

Since $\norm{\rho(t)}/g(t)\to0$ and $\gamma,\bar{\gamma}>0$, we have $\forall \epsilon_1,\epsilon_2>0$, $\exists t_{\epsilon_1},t_{\epsilon_2}$ such that 
\begin{equation}
\begin{split}
\forall t>t_{\epsilon_1},\;\forall n,\quad &\innerprod{\rho(t)}{x_n}\le \|{\rho(t)}\|_2 B\le \epsilon_1\gamma g(t),\text{ and }\\
\forall t>t_{\epsilon_2},\;\forall n,\quad &\innerprod{\rho(t)}{x_n}\ge -\|{\rho(t)}\|_2 B\ge -\epsilon_2\bar{\gamma} g(t)
\end{split}
\label{eq:rho}
\end{equation}

The first prove the following claim:
\begin{claim} For almost all $\{\x{n}\}$, $|S|<d$ and $\sigma_{|S|} (X_{S}) >0$, where $\sigma_k (A)$ is the $k^{th}$ singular value of $A$.
\end{claim}
\begin{proof} 
Since,  $S= \{n:\innerprod{\bar{\w}_\infty}{\x{n}}=\gamma\}$, we have $X_{S} \bar{\w}_\infty = \gamma 1_{S}\in\bR^{|S|}.$ 

If $X$ is randomly drawn from a continuous distribution,  for any fixed subset $J$ if $|J| >d$, the column span of $X_J$ is rank deficient and will miss any fixed vector $v$ that is independent of $X$ with probability $1$. Thus,
\begin{align}
\bR^{|J|}\ni1_{J} \notin \text{colspan}(X_J),\text{ for almost all }X_J\in\bR^{|J|\times d}.
\end{align}

Since we always have $1_{S} \in \text{colspan} (X_S)$, this implies for almost all $X$,  $|S| \le d$ and $\sigma_{|S|} (X_{S} )>0$.
\end{proof}

\paragraph{Exponential loss: }
For exponential loss, the gradient at   $\W{t}$  is given by
\begin{align}
\nonumber -\nabla \c{L}( \W{t} ) &= \sum_{n \in S} \exp(-\gamma g(t)) \exp(- \rho(t) ^\top x_n) x_n + \sum_{n \in S^c} \exp( - \bar \gamma_ng(t)) \exp(-\rho(t)^\top x_n) x_n\\
&:= I(t) + II(t),
\label{eq:1p2}
\end{align}
where $I(t)=\sum_{n \in S} \exp(-\gamma g(t)) \exp(- \rho(t) ^\top x_n) x_n $ and $II(t)=\sum_{n \notin S} \exp( - \bar \gamma_ng(t)) \exp(-\rho(t)^\top x_n) x_n$. 

To prove the lemma, we need to show that the gradient are dominated by the positive span of support vectors. Towards this goal, we will now show that $\lim\limits_{t\to\infty} \frac{\norm{II(t)}}{\norm{I(t)}} =0$. 

Recall that $\alpha(t)=[\alpha_n(t)]_n$ is defined as $\alpha_n(t) = \exp(-\innerprod{\rho(t)}{x_n})$ and $\alpha_S(t)\in\bR^{|S|}$ is a subvector restricted to indices in $S$. The following are true for any $\epsilon_1,\epsilon_2>0$.
\begin{asparaenum}[Step 1.]
\item \textit{Lower bound on $I(t)$: } There exists $t_{\epsilon_1}$ such that for all $t>t_{\epsilon_1}$, we have
\begin{align}
\nonumber \norm{I}_2 &= \exp(-\gamma g(t) )  \norm{X_{S} \alpha_S(t)} _2
\ge \exp(-\gamma g(t)) \sigma_{|S|} (X_S) \norm{\alpha_S(t)}_2\\
\nonumber&\ge \exp(-\gamma g(t)) \sigma_{|S|} (X_S) \max_{n\in S} \alpha_n(t)\\
&\overset{(a)}\ge\sigma_{|S|} (X_S) \exp(- (1+\epsilon_1)\gamma g(t)):= C_1 \exp(- (1+\epsilon_1)\gamma g(t)),
\label{eq:1}
\end{align}
where $(a)$ follows   from \eqref{eq:rho}, from which we get $\alpha_n(t)=\exp(-\innerprod{\rho(t)}{x_n})\ge\exp(- \epsilon_1\gamma g(t)) $, and $C_1=\sigma_{|S|} (X_S)>0$ is a constant independent of $t$. 

\item \textit{Upper bound on $II(t)$: } Again, for large enough  $t>t_{\epsilon_2}$, we have
\begin{align}
\nonumber \norm{II(t)}_2 &=\sum_{n \notin S} \exp( - \bar \gamma_ng(t)) \exp(-\rho(t)^\top x_n) x_n\le  N\max_n\exp(- \bar \gamma_n g(t)) \alpha_n\|x_n\|_2\\
\nonumber&\overset{(a)}\le \exp(- \bar \gamma g(t))BN \max_n\alpha_n\\
&\overset{(b)}\le BN\exp(-(1-\epsilon_2)\bar{\gamma} g(t)):=C_2 \exp(-(1-\epsilon_2)\bar{\gamma}g(t)),
\label{eq:2}
 \end{align}
 where $(a)$ uses $\forall n\notin S, \bar \gamma_n\ge \bar\gamma$ (recall that $\bar\gamma$ is the second smallest margin to $\bar{w}_\infty$) and $(b)$ follows  from \eqref{eq:rho}, using $\alpha_n=\exp(-\innerprod{\rho(t)}{x_n})\le\exp(\epsilon_2\bar{\gamma}g(t))$, and $C_2=BN>0$ is again a constant independent of $t$.

 \item \textit{Remaining steps in the proof: }  By combining \eqref{eq:1} and \eqref{eq:2} using $\epsilon_1=\nicefrac{(\bar{\gamma}-\gamma)}{4\gamma}$ and $\epsilon_2=\nicefrac{(\bar{\gamma}-\gamma)}{4\bar{\gamma}}$ and an appropriate constant $C>0$, we have for any norm $\|.\|$ 
 \begin{align}
 \frac{\norm{II(t)}}{\norm{I(t)}} &\le C\exp( - \frac12(\bar \gamma -\gamma)g(t))\overset{(a)}\to 0,
 \end{align}
 where $(a)$ follows from $\bar{\gamma}>\gamma$ and $g(t)=\|\W{t}\|\to \infty$. 
 
Finally, note that $ -\frac{\nabla \cL(\W{t})}{\norm{\nabla \cL(\W{t})}} = \frac{I(t)}{\norm{I(t)+II(t)}} +\frac{II(t)}{\norm{I(t)+II(t)}}.$
 Since $\norm{\frac{II(t)}{\norm{I(t)+II(t)}} }\le \frac{\norm{II(t)}/{\|I(t)\|}}{1-\norm{II(t)}/\norm{I(t)}}\overset{t\to\infty}\to0$, and $I(t) \propto \sum_{n\in S}\alpha_{n}(t)\x{n}$ for  $\alpha_n(t)>0$, we have shown that every limit point of $-\frac{\nabla \cL( \W{t})}{\norm{\nabla \cL( \W{t})}} \to \sum_{n\in S}\alpha_n\x{n}$ for some $\alpha_n>0$. 
 \end{asparaenum}
 Recall that in the beginnning of the proof we made a change of variable that $\x{n}\gets \y{n}\x{n}$. Reversing this change of variable finishes the proof for exponential loss.
\remove{\paragraph{Extension to general exponential tail losses}
For general exponential tail loss (Property~\ref{ass:exp-tail}), we can essentially flow the same steps with some additional book keeping. 

\begin{align}
\nonumber -\nabla \c{L}( \W{t} ) &= \sum_{n \in S} (-\ell^\prime(\innerprod{\W{t}}{x_n}) x_n+ \sum_{n \in S^c} (-\ell^\prime(\innerprod{\W{t}}{x_n}) x_n\\
&= I(t) + II(t).
\label{eq:1p2-1}
\end{align}

\textit{Step 1: Lower bound on $I(t)$: } For large enough  $t>t_{\epsilon_1}$, we have
\begin{align}
\nonumber \norm{I}_2 &= \exp(-\gamma g(t) )  \norm{X_{S} \alpha_S(t)} _2\ge \exp(-\gamma g(t)) \sigma_{|S|} (X_S) \norm{\alpha_S(t)}_2\ge \exp(-\gamma g(t)) \sigma_{|S|} (X_S) \min_n \alpha_n(t)\\
&\overset{(a)}\ge\sigma_{|S|} (X_S) \exp(- (1+\epsilon_1)\gamma g(t))\ge C_1 \exp(- (1+\epsilon_1)\gamma g(t)),
\end{align}
where $(a)$ follows from the definition of $\alpha_n=\exp(-\innerprod{\rho(t)}{x_n})$ and \eqref{eq:rho}, and $C_1>0$ is a constant independent of $t$. }
\end{proof}

\remove{
\subsection{Convergence of $\U{t}$ in direction}
The following general lemma shows the first part of Theorem~\ref{thm:mf-exp} on convergence of $\U{t}$. 
\begin{lemma} [$\U{t}$ converges in direction if $\Delta\U{t}$ converges in direction] Let $\U{t}$ be any diverging sequence of iterates, i.e., $\|\U{t}\|\to \infty$, iteratively defined using  bounded discrete updates of $\U{t+1}=\U{t}+\eta_t\Delta\U{t}$ for any $0<\eta_t<\infty$, any finite incremental update direction $\Delta\U{t}$ such that $\|\Delta\U{t}\|>0$, and any finite initialization $\U{0}<\infty$. 
 
For such sequences, if $\lim\limits_{t\to\infty}\frac{\Delta\U{t}}{\norm{\Delta\U{t}}}=\bar{U}_\infty$, then $\lim\limits_{t\to\infty}\frac{\U{t}}{\norm{\U{t}}}=\bar{U}_\infty$. 
\end{lemma}
\label{lem:wconv}
\begin{proof}
 For all $\eta_t>0$, since $\frac{\Delta\U{t}}{\norm{\Delta\U{t}}}\to\bar{U}_\infty$, we can write $\Delta\U{t}=\bar{U}_\infty h(t)+\xi(t)$ where $h(t)=\|\Delta\U{t}\|$ and $\frac{\xi(t)}{h(t)}\to0$. 
Also, define $g(t)$ and $\rho(t)$ as follows\footnote{Note that  $g(t)=\sum_{u<t}\norm{\Delta{\U{t}}}$ here is different from notation in the proof of Lemma~\ref{lem:grad-conv} where we defined $g(t)=\norm{\W{t}}$.}:
\begin{equation}
g(t):=\sum_{u<t}\eta_{u}h(u)\text{, and }\rho(t):=\U{t}-\bar{U}_\infty g(t)=\sum_{u<t}\eta_{u}\xi(u)+\U{0}.
\label{eq:hgt}
\end{equation} 


In order to prove the lemma, we need to show that $\frac{{\rho}(t)}{g(t)}\to 0$. We observe the following, 
\begin{asparaenum}
\item We have $\norm{\U{t}}-\|\U{0}\|\le \sum_{u<t}\eta_{u} \|\Delta\U{u}\|=g(t)$. Since, $\|\U{t}\|\to \infty$, we thus get $g(t) \to \infty.$
\item 
Also, $g(t)=\sum_{u<t}\eta_uh(u)$ is a strictly monotonically increasing  as $\forall t<\infty, h(t)=\norm{\Delta \U{t}}>0$. 
\item Finally, $\lim\limits_{t\to\infty} \frac{\rho(t+1)-\rho(t)}{g(t+1)-g(t)}=\lim\limits_{t\to\infty} \frac{\xi(t)}{ h(t)}=0.$
\end{asparaenum}
Thus, by using the Stolz-Cesaro (Theorem~\ref{thm:stolzcesaro}), we get $\lim\limits_{t\to\infty }\frac{\rho(t)}{g(t)}=0$.
\end{proof}
\remove{
Recall the set of global minima for $\c{L}(\w)$, $\c{G}=\{w: \forall n, \y{n}\innerprod{w}{\x{n}}> 0\text{ and }\|w\|=\infty\}$ and the set of $1$-margin solutions $\c{G}^{(1)}=\{w:\forall n, \y{n}\innerprod{w}{\x{n}}\ge 1\}$
\begin{proposition} If $\lim\limits_{t\to\infty}\c{L}(\W{t})=0$ and Lemma~\ref{lem:d-conv} holds, then $\exists \gamma>0$, such that $\frac{\w_{\infty}}{\gamma}\in \c{G}^{(1)}$.
\end{proposition} 
\begin{proof}
 Since $\lim_{t\to\infty}\c{L}(\W{t})=0$, we have $\forall n$, $\lim_{t\to\infty}\y{n}\innerprod{\W{t}}{\x{n}}=\infty>0$,   i.e., $\forall M>0,\exists t_M\st \forall t>t_M$, $\y{n}\innerprod{\W{t}}{\x{n}}\ge M$ for all $n$. This implies, for all $t>t_M$, $\y{n}\innerprod{\frac{\W{t}}{\|\W{t}\|}}{\x{n}}>0$, and hence under Lemma~\ref{lem:d-conv}, $\lim_{t\to\infty}\y{n}\innerprod{\frac{\W{t}}{\|\W{t}\|}}{\x{n}}=\innerprod{\w_\infty}{\x{n}}>0$. For any finite $N$, $\gamma:=\min_n \innerprod{w_\infty}{\x{n}}>0$, then $\w_\infty/\gamma\in\c{G}^{(1)}$.
\end{proof}
}
}

\subsection{Proof of Theorem~\ref{thm:mf-exp}}
{\thmmfexp*}
\begin{proof}
In this proof, $\norm{.}_F$, $\norm{.}_*$, and $\norm{.}_\text{op}$ denote the Frobenious norm, nuclear norm, and operator norm, respectively.

From the assumption of theorem, we have that $\U{t}$ converges in direction.  Let $\bar{U}_\infty=\lim\limits_{t\to\infty}\frac{\U{t}}{\norm{\U{t}}_F}$. Noting that for $\WW{t}=\U{t}\U{t}^\top$, $\norm{\WW{t}}_*=\norm{\U{t}}_F^2$, we have that $\lim\limits_{t\to\infty}\frac{\WW{t}}{\norm{\WW{t}}_*}=\lim\limits_{t\to\infty}\frac{\U{t}}{\norm{\U{t}}_F}\frac{\U{t}^\top}{\norm{\U{t}}_F}=\bar{U}_\infty\bar{U}_\infty^\top$. Denote $\bar{W}_\infty=\lim\limits_{t\to\infty}\frac{\WW{t}}{\norm{\WW{t}}_*}=\bar{U}_\infty\bar{U}_\infty^\top$.

Since $\WW{t}$ minimizes a strictly monotone loss, we have that $\|\WW{t}\|_*\to\infty$ and $\forall n, y_n\innerprod{\bar W_\infty}{X_n}>0$. Let $\gamma=\min_n y_n\innerprod{\bar W_\infty}{X_n}$ denote the margin of $\bar W_\infty$ and $S=\{n:\y{n}\innerprod{\bar W_\infty}{X_n}=\gamma\}$ denote the indices of the support vectors of ${\bar{W}}_\infty$. 

In order to prove the theorem, we can can equivalently  show that a positive scaling of $\bar U_\infty$ given by  $\bar{\bar{U}}_\infty=\bar{U}_\infty/\sqrt{\gamma}$ is the first order stationary point of eq. \eqref{eq:uopt}.  

In the remainder of the proof we show that $\bar{\bar{U}}_\infty$ satisfies the following KKT optimality conditions of \eqref{eq:uopt}:  
\begin{flalign}
\textbf{To show:}&&& \y{n}\innerprod{\bar{\bar{U}}_\infty\bar{\bar{U}}_\infty^\top}{X_n} \ge 1 \text{ and }\exists \alpha\ge0 \st &&\textbf{(primal and dual feasibility)}&\label{eq:pf}\\
&&& \forall i\notin {S}: \alpha_n=0\text{ and }&&\textbf{(complementary slackeness) }&\label{eq:cs}\\
&&& \bar{\bar{U}}_\infty=\sum_n\alpha_n y_nX_{n}{\bar{\bar{U}}}_\infty.&&\textbf{(stationarity) }&\label{eq:s}
\end{flalign}

\paragraph{Primal feasibility} This holds by definition since  $\bar{\bar{U}}_\infty\bar{\bar{U}}_\infty^\top=\bar{W}_\infty/\gamma$  has unit margin by the scaling.

\paragraph{Dual feasibility and complementary slackness} 
Denote $Z_{(t)}=-\nabla\c{L}(\WW{t})=\sum_n\exp(-y_n\innerprod{\WW{t}}{X_n}) y_n X_n$. From the assumptions in the theorem, we have that $Z_{(t)}$ converge in direction. Let  $\bar{Z}_\infty=\lim\limits_{t\to\infty}\frac{\Z{t}}{\norm{\Z{t}}_\text{op}}$. 
In addition, we also assume that $\c{L}(\WW{t})\to0$ and that  $\U{t}$ convergence in direction, which in turn implies convergence in direction of $\WW{t}=\U{t}\U{t}^\top$. Thus, from Lemma~\ref{lem:grad-conv}, we have  $\bar{Z}_\infty=\sum_{n\in S} {\alpha}_ny_nX_n$ for some $\{{\alpha}_n\}_{n=1}^N$ such that ${\alpha}_n\ge 0$ and ${\alpha}_n=0$ for all $n\notin S$. We propose this $\{\alpha_n\}_{n=1}^N$ as our candidate dual certificate, which satisfies both dual feasibility and complementary slackness.

\textbf{Stationarity: }To prove the theorem, we now need to show that: $\bar{\bar{U}}_\infty= D\bar Z_\infty \bar{\bar{U}}_\infty$, for some positive scalar $D$, or equivalently that ${\bar{U}}_\infty= D\bar Z_\infty {\bar{U}}_\infty$. This forms the main part of the proof. 

Using the assuptions in the theorem, we have that $\U{t}$ and $Z_{(t)}$  converges in direction, we introducing the following notation to conveniently represent these quantities.  
\begin{compactenum}
\item Since  $\frac{\U{t}}{\|\U{t}\|_F}\to\bar U_\infty$, we define $g(t)$ and $\rho_{(t)}$ satisfying the following,
\begin{equation}
\U{t}=\bar U_\infty g(t)+\rho_{(t)}
\;\st\; g(t):=\norm{\U{t}}_F\to\infty\text{ and }\frac{\rho_{(t)}}{g(t)}\to 0.
\label{eq:u}
\end{equation}
\item For exponential loss, $\c{L}(\WW{t})\to0$ implies $Z_{(t)}=-\nabla\c{L}(\WW{t})\to 0$. 
Thus, using the previously introduced notation $\bar Z_\infty=\lim\limits_{t\to\infty}\frac{Z_{(t)}}{\norm{Z_{(t)}}_\text{op}}$, we define ${p(t)}$ and $\zeta_{(t)}$ as follows
\begin{equation}
Z_{(t)}=-\nabla\c{L}(\WW{t})=\bar Z_\infty p(t)+\zeta_{(t)} \;\st\; p(t):=\norm{Z_{(t)}}_\text{op}\to 0 \text{ and }\frac{\zeta_{(t)}}{p(t)}\to0 .
\label{eq:g}
\end{equation}
\end{compactenum}

To show stationarity, we need to show that $\bar U_\infty=D\bar  Z_\infty \bar U_\infty$, which requires that the columns of $\bar W_\infty$ are spanned subset of eigenvectors of $\bar Z_\infty$ that correspond to the same eigen value. 

Let $\Delta\U{t}=\U{t+1}-\U{t}$. Substituting expressions of $\U{t}$ and $Z_{(t)}$ from \eqref{eq:u} and \eqref{eq:g}, respectively, for the updates $\Delta\U{t}$ from eq. \eqref{eq:mf-updateU}, we have 
\begin{flalign}
\nonumber \Delta \U{t}&=\eta_t Z_{(t)}\U{t}=\eta_t p(t) g(t)\left[\bar Z_\infty \bar U_\infty+ \bar Z_\infty\frac{\rho_{(t)}}{g(t)}+\frac{\zeta_{(t)}}{p(t)}\bar U_\infty\right]&\\
&\overset{(a)}=\eta_t p(t) g(t)[\bar Z_\infty \bar U_\infty+\delta_{(t)}]&
\label{eq:last}
\end{flalign}
where in $(a)$ we collect all the diminishing terms into $\delta_{(t)}=\bar Z_\infty\frac{\rho_{(t)}}{g(t)}+\frac{\zeta_{(t)}}{p(t)}\bar U_\infty \to 0$ as from eqs. \eqref{eq:u}--\eqref{eq:g}, we have $\frac{\rho_{(t)}}{g(t)},\frac{\zeta_{(t)}}{p(t)}\to0$ and $\bar Z_\infty$ and $\bar{W}_\infty$ are finite quanitities independent of $t$. 

Summing over $t$, we have that 

\begin{flalign}
\U{t}-\U{0}&=\bar Z_\infty \bar U_\infty\sum_{u<t}\eta_u p(u) g(u)+\sum_{u<t}\delta_{(u)}\eta_u p(u) g(u)
\label{eq:last1}
\end{flalign}

\begin{claim}$\norm{\bar Z_\infty \bar U_\infty}>0$ and $\sum_{u<t}\eta_u p(u) g(u)\to\infty$.
\end{claim}
\begin{proof}
First, recall that for the limit direction $\bar{W}_\infty=\bar U_\infty \bar U_\infty^\top$, $\min_ny_n\innerprod{\bar{W}_\infty}{X_n}=\gamma>0$ and   $\bar Z_\infty=\sum_{n\in S}\alpha_n y_n X_n$ for $\alpha_n\ge0$.
 Thus, $\innerprod{Z_\infty}{\bar W_\infty}=\innerprod{\bar Z_\infty \bar U_\infty}{\bar U_\infty}=\sum_{n\in S}\alpha_n \gamma >0$ for $\bar U_\infty\neq 0$, and hence $\norm{\bar Z_\infty \bar U_\infty}>0$.
 
 Secondly,  since $\delta_{(t)}\to0$ in eq. \eqref{eq:last1}, $\exists t_0$ such that $\forall t>t_0$, $\norm{\delta_{(t)}}\le 1$ and since all the incremental updates to gradient descent are finite, we have that $\sup_{t}\norm{\delta_{(t)}}<\infty$. Additionally,  since $p(t)=\norm{\Z{t}}_\text{op}$ and $g(u)=\norm{\U{t}}_F$ are positive, we have that $b_t=\sum_{u<t}\eta_u p(u) g(u)$ is monotonic increasing, thus if $\lim\sup_{t\to\infty} b_t=\infty$ then $\lim_{t\to\infty} b_t=\infty$. 
 On contrary, if $\lim\sup_{t\to\infty} b_t=C<\infty$, then we have from eq. \eqref{eq:last1}, $\norm{\U{t}}\le\norm{\U{0}}+\norm{\bar Z_\infty\bar U_\infty}C +\left(\sup_{t}\norm{\delta_{(t)}}\right) C<\infty$ which is a contradiction to $\norm{\U{t}}\to\infty$.
\end{proof}

From the above claim, we have that  the sequence $b_t=\sum_{u<t}\eta_u p(u) g(u)$ is monotonic increasing and diverging. Thus, for $a_t=\sum_{u<t}\delta_{(u)}\eta_u p(u) g(u)$, using Stolz-Cesaro theorem~(Theorem~\ref{thm:stolzcesaro}), we have that 
\begin{flalign}
\nonumber\lim_{t\to\infty}\frac{a_t}{b_t}=\lim_{t\to\infty}\frac{\sum_{u<t}\delta_{(u)}\eta_u p(u) g(u)}{\sum_{u<t}\eta_u p(u) g(u)}=\lim_{t\to\infty}\frac{a_{t+1}-a_{t}}{b_{t+1}-b_{t}}=\lim_{t\to\infty}\delta_{(t)}=0.\\
\implies \text{for }\tilde{\delta}_{(t)}\to0,\; \text{we have }\sum_{u<t}\delta_{(u)}\eta_u p(u) g(u)=\tilde{\delta}_{(t)}\sum_{u<t}\eta_u p(u) g(u),
\label{eq:last2}
\end{flalign}

Substituting eq. \eqref{eq:last2} in eq. \eqref{eq:last1}, we have 
\begin{flalign}
\U{t}&\overset{(a)}=\left[\bar Z_\infty \bar U_\infty+{\delta}'_{(t)}\right]\left[\sum_{u<t}\eta_u p(u) g(u)\right]
\end{flalign}
\begin{flalign}
\implies \frac{\U{t}}{\norm{\U{t}}}&=\frac{\bar Z_\infty \bar U_\infty+{\delta}'_{(t)}}{\norm{\bar Z_\infty \bar U_\infty+\delta'_{(t)}}_F}\overset{(a)}\to \frac{\bar Z_\infty \bar U_\infty}{\norm{\bar Z_\infty \bar U_\infty}}\\
\implies \bar U_\infty&=\lim_{t\to\infty}\frac{\U{t}}{\norm{\U{t}}}=\frac{1}{\norm{\bar Z_\infty \bar U_\infty}}\bar Z_\infty \bar U_\infty,
\end{flalign}
where in $(a)$ we absorbed all the diminishing terms into  $\delta'_{(t)}=\tilde{\delta}_{(t)}+\U{0}/\sum_{u<t}\eta_up(u)g(u)\to0$ and $(b)$ follows since $\bar Z_\infty \bar U_\infty\neq 0$  and hence dominates $\tilde{\delta}_{(t)}$.

We have thus shown that $\bar U_\infty=D\bar Z_\infty \bar U_\infty$ for $D=\frac{1}{\norm{\bar Z_\infty \bar U_\infty}}$ which completes the proof of the theorem.
\end{proof}

%% file: exp-loss-implicit-reg-supp.bbl
\begin{thebibliography}{35}
\providecommand{\natexlab}[1]{#1}
\providecommand{\url}[1]{\texttt{#1}}
\expandafter\ifx\csname urlstyle\endcsname\relax
  \providecommand{\doi}[1]{doi: #1}\else
  \providecommand{\doi}{doi: \begingroup \urlstyle{rm}\Url}\fi

\bibitem[Amari(1998)]{amari1998natural}
S.~I. Amari.
\newblock Natural gradient works efficiently in learning.
\newblock \emph{Neural computation}, 1998.

\bibitem[Bartlett and Mendelson(2003)]{bartlett2003rademacher}
P.~L. Bartlett and S.~Mendelson.
\newblock Rademacher and {Gauss}ian complexities: Risk bounds and structural
  results.
\newblock \emph{Journal of Machine Learning Research}, 2003.

\bibitem[Beck and Teboulle(2003)]{beck2003mirror}
A.~Beck and M.~Teboulle.
\newblock Mirror descent and nonlinear projected subgradient methods for convex
  optimization.
\newblock \emph{Operations Research Letters}, 2003.

\bibitem[Boyd and Vandenberghe(2004)]{boyd2004convex}
S.~Boyd and L.~Vandenberghe.
\newblock \emph{Convex optimization}.
\newblock Cambridge university press, 2004.

\bibitem[Bregman(1967)]{bregman1967relaxation}
L.~M. Bregman.
\newblock The relaxation method of finding the common point of convex sets and
  its application to the solution of problems in convex programming.
\newblock \emph{USSR computational mathematics and mathematical physics}, 1967.

\bibitem[Candes and Recht(2009)]{candes2009exact}
E.~J. Candes and B.~Recht.
\newblock Exact matrix completion via convex optimization.
\newblock \emph{Foundations of Computational Mathematics}, 2009.

\bibitem[Duchi et~al.(2011)Duchi, Hazan, and Singer]{duchi2011adaptive}
J.~Duchi, E.~Hazan, and Y.~Singer.
\newblock Adaptive subgradient methods for online learning and stochastic
  optimization.
\newblock \emph{Journal of Machine Learning Research}, 2011.

\bibitem[Efron et~al.(2004)Efron, Hastie, Johnstone, and
  Tibshirani]{efron2004least}
B.~Efron, T.~Hastie, I.~Johnstone, and R.~Tibshirani.
\newblock Least angle regression.
\newblock \emph{The Annals of statistics}, 2004.

\bibitem[Friedman(2001)]{friedman2001greedy}
Jerome~H Friedman.
\newblock Greedy function approximation: a gradient boosting machine.
\newblock \emph{Annals of statistics}, 2001.

\bibitem[Gunasekar et~al.(2017)Gunasekar, Woodworth, Bhojanapalli, Neyshabur,
  and Srebro]{gunasekar2017implicit}
Suriya Gunasekar, Blake~E Woodworth, Srinadh Bhojanapalli, Behnam Neyshabur,
  and Nati Srebro.
\newblock Implicit regularization in matrix factorization.
\newblock In \emph{Advances in Neural Information Processing Systems}, pages
  6152--6160, 2017.

\bibitem[Hoffer et~al.(2017)Hoffer, Hubara, and Soudry]{Hoffer2017}
Elad Hoffer, I~Hubara, and D.~Soudry.
\newblock {Train longer, generalize better: closing the generalization gap in
  large batch training of neural networks}.
\newblock In \emph{NIPS}, pages 1--13, may 2017.
\newblock URL \url{http://arxiv.org/abs/1705.08741}.

\bibitem[Keskar et~al.(2016)Keskar, Mudigere, Nocedal, Smelyanskiy, and
  Tang]{keskar2016large}
Nitish~Shirish Keskar, Dheevatsa Mudigere, Jorge Nocedal, Mikhail Smelyanskiy,
  and Ping Tak~Peter Tang.
\newblock On large-batch training for deep learning: Generalization gap and
  sharp minima.
\newblock In \emph{International Conference on Learning Representations}, 2016.

\bibitem[Kingma and Adam(2015)]{kingma2015adam}
D~Kingma and Jimmy~Ba Adam.
\newblock Adam: A method for stochastic optimisation.
\newblock In \emph{International Conference for Learning Representations},
  volume~6, 2015.

\bibitem[Kivinen and Warmuth(1997)]{kivinen1997exponentiated}
Jyrki Kivinen and Manfred~K Warmuth.
\newblock Exponentiated gradient versus gradient descent for linear predictors.
\newblock \emph{Information and Computation}, 1997.

\bibitem[Li et~al.(2017)Li, Ma, and Zhang]{li2017algorithmic}
Yuanzhi Li, Tengyu Ma, and Hongyang Zhang.
\newblock Algorithmic regularization in over-parameterized matrix recovery.
\newblock \emph{arXiv preprint arXiv:1712.09203}, 2017.

\bibitem[Muresan and Muresan(2009)]{muresan2009concrete}
Marian Muresan and Marian Muresan.
\newblock \emph{A concrete approach to classical analysis}, volume~14.
\newblock Springer, 2009.

\bibitem[Nemirovskii and Yudin(1983)]{nemirovskii1983problem}
A.~Nemirovskii and D.~Yudin.
\newblock \emph{Problem complexity and method efficiency in optimization}.
\newblock Wiley, 1983.

\bibitem[Nesterov(1983)]{nesterov1983method}
Yurii Nesterov.
\newblock A method of solving a convex programming problem with convergence
  rate o (1/k2).
\newblock In \emph{Soviet Mathematics Doklady}, 1983.

\bibitem[Neyshabur et~al.(2015{\natexlab{a}})Neyshabur, Salakhutdinov, and
  Srebro]{neyshabur2015path}
Behnam Neyshabur, Ruslan~R Salakhutdinov, and Nati Srebro.
\newblock Path-sgd: Path-normalized optimization in deep neural networks.
\newblock In \emph{Advances in Neural Information Processing Systems}, pages
  2422--2430, 2015{\natexlab{a}}.

\bibitem[Neyshabur et~al.(2015{\natexlab{b}})Neyshabur, Tomioka, and
  Srebro]{neyshabur2015search}
Behnam Neyshabur, Ryota Tomioka, and Nathan Srebro.
\newblock In search of the real inductive bias: On the role of implicit
  regularization in deep learning.
\newblock In \emph{International Conference on Learning Representations},
  2015{\natexlab{b}}.

\bibitem[Neyshabur et~al.(2017)Neyshabur, Tomioka, Salakhutdinov, and
  Srebro]{neyshabur2017geometry}
Behnam Neyshabur, Ryota Tomioka, Ruslan Salakhutdinov, and Nathan Srebro.
\newblock Geometry of optimization and implicit regularization in deep
  learning.
\newblock \emph{arXiv preprint arXiv:1705.03071}, 2017.

\bibitem[Polyak(1964)]{polyak1964some}
Boris~T Polyak.
\newblock Some methods of speeding up the convergence of iteration methods.
\newblock \emph{USSR Computational Mathematics and Mathematical Physics}, 1964.

\bibitem[Recht et~al.(2010)Recht, Fazel, and Parrilo]{recht2010guaranteed}
B.~Recht, M.~Fazel, and P.~A. Parrilo.
\newblock Guaranteed minimum-rank solutions of linear matrix equations via
  nuclear norm minimization.
\newblock \emph{SIAM review}, 2010.

\bibitem[Ross(1980)]{ross1980elementary}
Kenneth~A Ross.
\newblock \emph{Elementary analysis}.
\newblock Springer, 1980.

\bibitem[Rosset et~al.(2004)Rosset, Zhu, and Hastie]{rosset2004boosting}
S.~Rosset, J.~Zhu, and T.~Hastie.
\newblock Boosting as a regularized path to a maximum margin classifier.
\newblock \emph{Journal of Machine Learning Research}, 2004.

\bibitem[Rudin et~al.(2004)Rudin, Daubechies, and Schapire]{rudin2004dynamics}
Cynthia Rudin, Ingrid Daubechies, and Robert~E Schapire.
\newblock The dynamics of adaboost: Cyclic behavior and convergence of margins.
\newblock \emph{Journal of Machine Learning Research}, 5\penalty0
  (Dec):\penalty0 1557--1595, 2004.

\bibitem[Schapire and Freund(2012)]{schapire2012boosting}
Robert~E Schapire and Yoav Freund.
\newblock \emph{Boosting: Foundations and algorithms}.
\newblock MIT press, 2012.

\bibitem[Shalev-Shwartz and Singer(2010)]{shalev2010equivalence}
Shai Shalev-Shwartz and Yoram Singer.
\newblock On the equivalence of weak learnability and linear separability: New
  relaxations and efficient boosting algorithms.
\newblock \emph{Machine learning}, 80\penalty0 (2-3):\penalty0 141--163, 2010.

\bibitem[Smith(2018)]{Smith2018}
Le~Smith, Kindermans.
\newblock {Don't Decay the Learning Rate, Increase the Batch Size}.
\newblock In \emph{ICLR}, 2018.

\bibitem[Soudry et~al.(2017)Soudry, Hoffer, and Srebro]{soudry2017implicit}
Daniel Soudry, Elad Hoffer, and Nathan Srebro.
\newblock The implicit bias of gradient descent on separable data.
\newblock \emph{arXiv preprint arXiv:1710.10345}, 2017.

\bibitem[Srebro et~al.(2005)Srebro, Alon, and
  Jaakkola]{srebro2005generalization}
Nathan Srebro, Noga Alon, and Tommi~S Jaakkola.
\newblock Generalization error bounds for collaborative prediction with
  low-rank matrices.
\newblock In \emph{Advances In Neural Information Processing Systems}, pages
  1321--1328, 2005.

\bibitem[Telgarsky(2013)]{telgarsky2013margins}
Matus Telgarsky.
\newblock Margins, shrinkage and boosting.
\newblock In \emph{Proceedings of the 30th International Conference on
  International Conference on Machine Learning-Volume 28}, pages II--307. JMLR.
  org, 2013.

\bibitem[Wilson et~al.(2017)Wilson, Roelofs, Stern, Srebro, and
  Recht]{wilson2017marginal}
Ashia~C Wilson, Rebecca Roelofs, Mitchell Stern, Nati Srebro, and Benjamin
  Recht.
\newblock The marginal value of adaptive gradient methods in machine learning.
\newblock In \emph{Advances in Neural Information Processing Systems}, pages
  4151--4161, 2017.

\bibitem[Zhang et~al.(2017)Zhang, Bengio, Hardt, Recht, and
  Vinyals]{zhang2017understanding}
Chiyuan Zhang, Samy Bengio, Moritz Hardt, Benjamin Recht, and Oriol Vinyals.
\newblock Understanding deep learning requires rethinking generalization.
\newblock In \emph{International Conference on Learning Representations}, 2017.

\bibitem[Zhang et~al.(2005)Zhang, Yu, et~al.]{zhang2005boosting}
Tong Zhang, Bin Yu, et~al.
\newblock Boosting with early stopping: Convergence and consistency.
\newblock \emph{The Annals of Statistics}, 33\penalty0 (4):\penalty0
  1538--1579, 2005.

\end{thebibliography}
